\documentclass{article}

\usepackage[final,nonatbib]{neurips_2020}

\usepackage[utf8]{inputenc} 
\usepackage[T1]{fontenc}    
\usepackage{hyperref}       
\usepackage{url}            
\usepackage{booktabs}       
\usepackage{amsfonts}       
\usepackage{nicefrac}       
\usepackage{microtype}      

\usepackage{booktabs}
\usepackage{multibib}
\newcites{supp}{Supplementary references}
\usepackage{setspace}
\usepackage{algorithm}
\usepackage[noend]{algpseudocode}
\algnewcommand{\IfThen}[2]{
  \State \algorithmicif\ #1\ \algorithmicthen\ #2}
\algdef{SE}[Class]{Class}{EndClass}[1]{\textbf{class} \textsc{#1}}{}
\usepackage[title]{appendix}
\usepackage{thmtools,thm-restate}
\usepackage{amsmath, enumerate}
\usepackage{amsthm}
\usepackage{graphicx}
\usepackage{tikz}
\usetikzlibrary{positioning}

\newtheorem{lemma}{Lemma}

\newtheorem{corollary}{Corollary}

{
\theoremstyle{plain}

}
\newenvironment{proofoutline}
 {\proof[Proof outline]}
 {\endproof}
\DeclareMathOperator{\E}{\mathbb{E}}
\DeclareMathOperator*{\argmin}{argmin}
\newcommand{\suchthat}{\;\ifnum\currentgrouptype=16 \middle\fi|\;}
\newcommand{\norm}[1]{\left\lVert #1 \right\rVert}
\newcommand{\Exp}[1]{\E \left[ #1 \right]}

\title{Adaptive Importance Sampling for Finite-Sum Optimization and Sampling with Decreasing Step-Sizes}

\author{%
  Ayoub El Hanchi\\
  McGill University\\
  \texttt{ayoub.elhanchi@mail.mcgill.ca} \\
  \And
  David A. Stephens\\
  McGill University\\
  \texttt{david.stephens@mcgill.ca} \\
}

\begin{document}

\maketitle

\begin{abstract}
  Reducing the variance of the gradient estimator is known to improve the convergence rate of stochastic gradient-based optimization and sampling algorithms. One way of achieving variance reduction is to design importance sampling strategies. Recently, the problem of designing such schemes was formulated as an online learning problem with bandit feedback, and algorithms with sub-linear \emph{static} regret were designed. In this work, we build on this framework and propose \emph{Avare}, a simple and efficient algorithm for adaptive importance sampling for finite-sum optimization and sampling with decreasing step-sizes. Under standard technical conditions, we show that \emph{Avare} achieves $\mathcal{O}(T^{2/3})$ and $\mathcal{O}(T^{5/6})$  \emph{dynamic} regret for SGD and SGLD respectively when run with $\mathcal{O}(1/t)$ step sizes. We achieve this dynamic regret bound by leveraging our knowledge of the dynamics defined by the algorithm, and combining ideas from online learning and variance-reduced stochastic optimization. We validate empirically the performance of our algorithm and identify settings in which it leads to significant improvements.
\end{abstract}

\section{Introduction}
Functions $f: \mathbb{R}^{d} \rightarrow \mathbb{R}$ of the form:
\begin{equation}
    f(x) = \sum_{i=1}^{N} f_i(x)
    \label{finite_sum}
\end{equation}
are prevalent in modern machine learning and statistics. Important examples include the empirical risk in the empirical risk minimization framework, \footnote{Up to a normalizing factor of $\frac{1}{N}$ which does not affect the optimization}
or the log-posterior of an exchangeable Bayesian model. When $N$ is large, the preferred methods for solving the resulting optimization or sampling problem usually rely on stochastic estimates of the gradient of $f$, using variants of  stochastic gradient descent (SGD) \cite{robbins1951} or stochastic gradient Langevin dynamics (SGLD) \cite{DBLP:conf/icml/WellingT11}:
\begin{align}
    x^{SGD}_{t+1} &= x^{SGD}_{t} - \alpha_t N \nabla f_{I_t}(x^{SGD}_t) \label{SGD} \\
    x^{SGLD}_{t+1} &= x^{SGLD}_t - \alpha_t N \nabla f_{I_t}(x^{SGLD}_t) + \xi_t \quad\quad \xi_t \sim \mathcal{N}(0, 2\alpha_t) \label{SGLD}
\end{align}
where $\{\alpha_t\}_{t=1}^{T}$ is a sequence of step-sizes, and the index $I_t$ is sampled uniformly from $[N]$, making $N \nabla f_{I_t}(x)$ an unbiased estimator of the gradient of $f$. We use $\{x_t\}_{t=1}^{T}$ to refer to either sequence when we do not wish to distinguish between them. It is well known that the quality of the answers given by these algorithms depends on the (trace of the) variance of the gradient estimator, and considerable efforts have been made to design methods that reduce this variance. 

In this paper, we focus on the importance sampling approach to achieving variance reduction. At each iteration, the algorithm samples $I_t$ according to a specified distribution $p^t$, and estimates the gradient using:
\begin{equation}
\label{estimator}
    \hat{g}^t := \frac{1}{p^t_{I_t}} \nabla f_{I_t}(x_t)
\end{equation}
It is immediate to verify that $\hat{g}^t$ is an unbiased estimator of $g^t := \nabla f(x_t)$. By cleverly choosing the distributions $p^t$, one can achieve significant variance reduction (up to a factor of N) compared to the estimator based on uniform sampling. Unfortunately, computing the variance-minimizing distributions at each iteration requires the knowledge of the Euclidean norm of all the individual gradients $g_i^t := \nabla f_i(x_t)$ at each iteration, making it unpractical \cite{DBLP:conf/nips/NeedellWS14, DBLP:conf/icml/ZhaoZ15}.

Many methods have been proposed that attempt to construct sequences of distributions $\{p^t\}_{t=1}^{T}$ that result in efficient estimators \cite{DBLP:conf/nips/NeedellWS14, DBLP:conf/icml/ZhaoZ15, DBLP:journals/corr/BouchardTPG15, DBLP:conf/nips/StichRJ17, DBLP:conf/icml/KatharopoulosF18, DBLP:conf/nips/JohnsonG18, DBLP:journals/jmlr/CsibaR18}.
Of particular interest to us, the task of designing such sequences was recently cast as an online learning problem with bandit feedback \cite{DBLP:conf/icml/NamkoongSYD17, DBLP:journals/corr/abs-1708-02544, DBLP:conf/colt/Borsos0L18, DBLP:conf/icml/BorsosCL019}. In this formulation, one attempts to design algorithms with sub-linear \emph{expected} static regret, which is defined as:
\begin{equation*}
    \text{Regret}_{S}(T) := \sum_{t=1}^{T} c_t(p^t) - \min_{p \in \Delta} \sum_{t=1}^{T} c_t(p)
\end{equation*}

where $\Delta$ denotes the probability simplex in $\mathbb{R}^N$, and $c_t(p)$ is the trace of the covariance matrix of the gradient estimator (\ref{estimator}), which is easily shown to be:
\begin{equation}
\label{cost}
    c_t(p) := \sum_{i=1}^{N} \frac{1}{p_i} \norm{g_i^t}_2^2 - \norm{g^t}_2^2
\end{equation}
Note that the second term cancels in the definition of regret, and we omit it in the rest of our discussion.
In this formulation, and to keep the computational load manageable, one has only access to partial feedback in the form of the norm of the $I_t^{th}$ gradient, and not to the complete cost function (\ref{cost}).
Under the assumption of uniformly bounded gradients, the best result in this category can be found in \cite{DBLP:conf/colt/Borsos0L18} where an algorithm with $\Tilde{O}(T^{2/3})$ static regret is proposed. A more difficult but more natural performance measure that makes the attempt to approximate the optimal distributions explicit is the \emph{dynamic} regret, defined by:
\begin{equation}
\label{dynmaicregret}
    \text{Regret}_{D}(T) := \sum_{t=1}^{T} c_t(p^t) - \sum_{t=1}^{T} \min_{p \in \Delta} c_t(p)
\end{equation}
Guarantees with respect to the expected dynamic regret are more difficult to obtain, and require that the cost functions $c_t(p)$ do not change too rapidly with respect to some reasonable measure of variation. See \cite{DBLP:journals/ior/BesbesGZ15, DBLP:journals/corr/abs-1708-03020, DBLP:conf/cdc/MokhtariSJR16, DBLP:conf/icml/YangZJY16, DBLP:conf/nips/ZhangLZ18} for examples of such measures and the corresponding regret bounds for general convex cost functions.

In this work, we propose \emph{Avare}, an algorithm that achieves sub-linear dynamic regret for both SGD and SGLD when the sequence of step-sizes $\{\alpha_t\}_{t=1}^{T}$ is decreasing. The name \emph{Avare} is derived from adaptive variance minimization. Specifically, our contributions are as follows:
\begin{itemize}
    \item We show that \emph{Avare} achieves $\mathcal{O}(T^{2/3})$ and $\mathcal{O}(T^{5/6})$ dynamic regret for SGD and SGLD respectively when $\alpha_t$ is $\mathcal{O}(1/t)$.
    \item We propose a new mini-batch estimator that combines the benefits of sampling without replacement and importance sampling while preserving unbiasedness.
    \item We validate empirically the performance of our algorithm and identify settings in which it leads to significant improvements.
\end{itemize}

We would like to point out that while the decreasing step size requirement might seem restrictive, we argue that for SGD and SGLD, it is the right setting to consider for variance reduction. Indeed, it is well known that under suitable technical conditions, both algorithms converge to their respective solutions exponentially fast in the early stages. Variance reduction is primarily useful at later stages when the noise from the stochastic gradient dominates. In the absence of control variates, one is forced to use decreasing step-sizes to achieve convergence. This is precisely the regime we consider.


\section{Related work}
\label{sec2}
It is easy to see that the cost functions (\ref{cost}) are convex over the probability simplex.
A celebrated algorithm for convex optimization over the simplex is entropic descent \cite{Beck2003}, an instance of mirror descent \cite{Nemirovsky1984} where the Bregman divergence is taken to be the relative entropy. A slight modification of this algorithm for online learning is the EXP3 algorithm \cite{DBLP:journals/siamcomp/AuerCFS02} which mixes the iterates of entropic descent with a uniform distribution to avoid the assignment of null probabilities. See \cite{DBLP:conf/icml/Zinkevich03, DBLP:journals/ftml/Shalev-Shwartz12, DBLP:journals/ftopt/Hazan16} for a more thorough discussion of online learning (also called online convex optimization) in general and variants of this algorithm in particular. 

Since we are working over the simplex, the EXP3 algorithm is a natural choice. This is the approach taken in \cite{DBLP:conf/icml/NamkoongSYD17} and \cite{DBLP:journals/corr/abs-1708-02544}, although strictly speaking neither is able to prove sub-linear static regret bounds. The difficulty comes from the fact that the norm of the gradients of the cost functions (\ref{cost}) explode to infinity on the boundary of the simplex. This is amplified by the use of stochastic gradients which grow as $1/p_{min}^{5}$ in expectation, making it very difficult to reach regions near the boundary. Algorithms based on entropic descent for dynamic regret problems also exist, including the fixed share algorithm and projected mirror descent \cite{HERBSTER1995286, 10.1162/153244301753683726, DBLP:conf/nips/Cesa-BianchiGLS12, DBLP:conf/icml/GyorgyS16}. Building on these algorithms, and by artificially making the cost functions strongly-convex to allow the use of decreasing step-sizes, we were only able to show $\Tilde{O}(T^{7/8})$ dynamic regret using projected mirror descent for SGD with $O(1/t)$ decreasing step sizes and uniformly bounded gradients.

The approach taken in \cite{DBLP:conf/colt/Borsos0L18} is qualitatively different and is based on the follow-the-regularized-leader (FTRL) scheme. By solving the successive minimization problems stemming from the FTRL scheme analytically, the authors avoid the above-mentioned issue of exploding gradients. The rest of their analysis relies on constructing an unbiased estimate of the cost functions (\ref{cost}) using only the partial feedback $\norm{g_{I_t}^t}_2^2$, and probabilistically bounding the deviation of the estimated cost functions from the true cost functions using a martingale concentration inequality. The final result is an algorithm that enjoys an $\tilde{O}(T^{2/3})$ static regret bound.

Our approach is similar to the one taken in \cite{DBLP:conf/colt/Borsos0L18} in that we work directly with the cost functions and not their gradients to avoid the exploding gradient issue. Beyond this point however, our analysis is substantially different. While we are still working within the online learning framework, our analysis is more closely related to the analysis of variance-reduced stochastic optimization algorithms that are based on control variates. In particular, we study a Lyapunov-like functional and show that it decreases along the trajectory of the dynamics defined by the algorithms. This yields simpler and more concise proofs, and opens the door for a unified analysis.

\section{Algorithm}
\label{sec3}
Most of the literature on online convex optimization with dynamic regret relies on the use of the gradients of the cost functions \cite{DBLP:journals/ior/BesbesGZ15, DBLP:journals/corr/abs-1708-03020, DBLP:conf/cdc/MokhtariSJR16, DBLP:conf/icml/YangZJY16, DBLP:conf/nips/ZhangLZ18}. However, as explained in the previous section, such approaches are not viable for our problem. Unfortunately, the regret guarantees obtained from the FTRL scheme for static regret used in \cite{DBLP:conf/colt/Borsos0L18} do not directly translate into guarantees for dynamic regret.

We start by presenting the high level ideas that go into the construction of our algorithm. We then present our algorithm in explicit form, and discuss its implementation and computational complexity.

\subsection{High level ideas}
The simplest update rule one might consider when working with dynamic regret is a natural extension of the follow-the-leader approach:
\begin{equation}
\label{FTLL}
    p^{t+1} := \argmin_{p \in \Delta} \left\{c_t(p)\right\}
\end{equation}
Intuitively, if the cost functions do not change too quickly, then it is reasonable to expect good performance from this algorithm. In our case however, we do not have access to the full cost function. The traditional way of dealing with this problem is to build unbiased estimates of the cost functions using the bandit feedback, and then bounding the deviation of these unbiased estimates from the true cost functions. While this might work, we consider here a different approach based on constructing surrogate cost functions that are not necessarily unbiased estimates of the true costs.

For each $i \in [N]$, denote by $h_i^t$ the last observed gradient of $f_i$ at time $t$, with $h_{i}^{1}$ initialized arbitrarily (to $0$ for example). We consider the following surrogate cost for all $t \in [T]$:
\begin{equation}
    \label{surrogatecost}
    \widetilde{c}_t(p) := \sum_{i=1}^{N} \frac{1}{p_i} \norm{h_i^t}_2^2
\end{equation}
As successive iterates of the algorithm become closer and closer to each other, the squared norm of the $h_i^t$s become better and better approximations to the squared norm of the $g_i^t$s, thereby making the surrogate costs (\ref{surrogatecost}) accurate approximations of the true costs (\ref{cost}). The idea of using previously seen gradients to approximate current gradients is inspired by the celebrated variance-reduced stochastic optimization algorithm SAGA \cite{DBLP:journals/corr/DefazioBL14}.

We are now almost ready to formulate our algorithm. If we try to directly minimize the surrogate cost functions over the entire simplex at each iteration as in (\ref{FTLL}), we might end up assigning null or close to null probabilities to certain indices. Depending on how far we are in the algorithm, this might or might not be a problem. In the initial stages, this is clearly an issue since this might only be an artifact of the initialization (notably when we initialize the $h_{i}^{1}$ to $0$), or an accurate representation of the current norm of the gradient, but which might not be representative later on as the algorithm progresses. On the other hand, in the later stages of the algorithm, the cost functions are nearly identical, so that an assignment of near zero probabilities is a reflection of the true optimal probabilities, and is not necessarily problematic.

The above discussion suggests the following algorithm. Define the $\varepsilon$-restricted probability simplex to be:
\begin{equation}
\label{restrictedsimplex}
    \Delta(\varepsilon) := \left\{p \in \mathbb{R}^N \mid p_i \geq \varepsilon, ~ \sum_{i=1}^{N}p_i = 1 \right\}
\end{equation}
And let $\{\varepsilon_t\}_{t=1}^{T}$ be a decreasing sequence of positive numbers with $\varepsilon_1 \leq \frac{1}{N}$. Then we propose the following algorithm:
\begin{equation}
\label{bestalgever}
    p^{t} := \argmin_{p \in \Delta(\varepsilon_{t})} \left\{\widetilde{c}_t(p) \right\}
\end{equation}
Our theoretical analysis in Section \ref{sec4} suggests a specific decay rate for the sequence $\{\varepsilon_t\}_{t=1}^{T}$ that depends on the sequence of step-sizes $\{\alpha_t\}_{t=1}^{T}$ and whether SGD or SGLD is run.

\subsection{Explicit form}
Equation (\ref{bestalgever}) defines our sequence of distribution $\{p^{t}\}_{t=1}^{T}$, but the question remains whether we can solve the implied optimization problems efficiently. In this section we answer this question in the affirmative, and provide an explicit algorithm for the computation of the sequence $\{p^{t}\}_{t=1}^{T}$.

We state our main result of this section in the following lemma. The proof can be found in appendix \ref{appendix1}.
\begin{restatable}{lemma}{solutionoptproblem}
\label{solutionoptproblem}
    Let $\{a_i\}_{i=1}^{N}$ be a non-negative set of numbers where at least one of the $a_i$s is strictly positive, and let $\varepsilon \in [0, 1/N]$. Let $\pi: [N] \rightarrow [N]$ be a permutation that orders $\{a_i\}_{i=1}^{N}$ in a decreasing order ($a_{\pi(1)} \geq a_{\pi(2)} \geq \dots \geq a_{\pi(N)}$). Define:
    \begin{equation}
        \label{threshold}
        \rho:= \max \left\{i \in [N] \suchthat a_{\pi(i)} \geq \varepsilon \frac{\sum_{j=1}^{i} a_{\pi(j)}}{1-(N-i)\varepsilon}  \right\}
    \end{equation}
    and:
    \begin{equation}
        \label{normalizationconstant}
        \lambda := \frac{\sum_{j=1}^{\rho} a_{\pi(j)}}{1-(N-\rho)\varepsilon}
    \end{equation}
    Then a solution of the optimization problem:
    \begin{equation}
    \label{optproblem}
        \min_{p \in \Delta(\varepsilon)} \sum_{i=1}^{N} \frac{1}{p_i} a_i^2
    \end{equation}
    is given by:
    \begin{equation}
    \label{explicitsol}
        p^{*}_i = \begin{cases}
        a_i/\lambda &\text{if $i \in \{\pi(1), \dots, \pi(\rho)\}$}\\
        \varepsilon &\text{otherwise}
        \end{cases}
    \end{equation}
    In the case all $a_i$ are zero, any $p \in \Delta(\varepsilon)$ is a solution.
\end{restatable}

In light of Lemma \ref{solutionoptproblem}, to compute $p^t$ as defined in (\ref{bestalgever}), it is enough to know the value of $\rho_t$ as defined in (\ref{threshold}), replacing $a_i$ with $\norm{h_i^t}_2$ and $\varepsilon$ with $\varepsilon_t$. Using $\rho_t$ we can then compute $\lambda_t$ using (\ref{normalizationconstant}), and obtain $p^t$ from (\ref{explicitsol}). It remains to specify an efficient way to perform this computation.

\subsection{Implementation details}
The naive way to perform the above computation is to do the following at each iteration:
\begin{itemize}
    \item Sort $\{\norm{h_i^t}_2\}_{i=1}^{N}$ in decreasing order.
    \item Find $\rho_t$ by traversing $(\pi(i))_{i=1}^{N}$ in increasing order and finding the first $i \in [N]$ for which the inequality in (\ref{threshold}) does not hold.
    \item Explicitly compute the probabilities using
    (\ref{normalizationconstant}) and (\ref{explicitsol}).
    \item Sample from $p^t$ using inverse transform sampling.
\end{itemize}
This has complexity $O(N\log{N})$. In appendix \ref{appendix1}, we present an algorithm that requires only $O(N)$ vectorized operations, $O(\log^2 N)$ sequential (non-vectorized) operations, and $cN$ memory for small $c$. The algorithm uses three data structures: 
\begin{itemize}
    \item An array storing $\{\norm{h_i^t}_2\}_{i=1}^{N}$ unsorted.
    \item An order statistic tree storing the pairs $(\norm{h_i^t}_2, i)_{i=1}^{N}$ sorted according to $\norm{h_i^t}_2$.
    \item An array storing $\{\sum_{j=1}^{i}\norm{h_{\pi(j)}^t}_2\}_{i=1}^{N}$ where $\pi$ is the permutation that sorts $\{\norm{h_i^t}_2\}_{i=1}^{N}$ in the order statistic tree.
\end{itemize}
The order statistic tree along with the array storing the cumulative sums allows the retrieval of $\rho_t$ in $O(\log^2 N)$ time. The cumulative sums allow to sample from $p^t$ in $O(\log N)$ time using binary search, and maintaining them is the only operation that requires a vectorized $O(N)$ operation. All other operations run in $O(\log N)$ time. See appendix \ref{appendix1} for the full version of the algorithm and a more complete discussion of its computational complexity.

\section{Theory}
\label{sec4}
In this section, we prove a sub-linear dynamic regret guarantee for our proposed algorithm when used with SGD and SGLD with decreasing step-sizes. We present our results in a more general setting, and show that they apply to our cases of interest. We start by stating our assumptions:
\begin{restatable}{assumption}{boundedgradients} 
\label{boundedgradients}
(Bounded gradients)
There exists a $G > 0$ such that $\norm{\nabla f_i(x)}_2 \leq G$ for all $x \in \mathbb{R}^d$ and for all $i \in [N]$.
\end{restatable}
\begin{restatable}{assumption}{smoothness} 
\label{smoothness}
(Smoothness)
There exists an $L > 0$ such that $\norm{\nabla f_i(x) - \nabla f_i(y)}_2 \leq L \norm{x - y}_2$ for all $x, y \in \mathbb{R}^d$ and for all $i \in [N]$.
\end{restatable}
\begin{restatable}{assumption}{contraction} 
\label{contraction}
(Contraction of the iterates in expectation)
There exists constants $A \geq 0$, $B \geq 1$ and $\delta \in (0,1]$ such that $\Exp{\norm{x_{t+1} - x_{t}}_2 \suchthat I_1, \dots, I_{t-1}} \leq A/(B+t-1)^{\delta}$ for all $t \in [T]$.
\end{restatable}
The bounded gradients assumption has been used in all previous related work \cite{DBLP:conf/icml/NamkoongSYD17, DBLP:journals/corr/abs-1708-02544, DBLP:conf/colt/Borsos0L18, DBLP:conf/icml/BorsosCL019}, although it is admittedly quite strong. The smoothness assumption is standard in the study of optimization and sampling algorithms. Note that we chose to state our assumptions using index-independent constants to make the presentation clearer and since this does not affect our derivation of the sequence $\{\varepsilon_t\}_{t=1}^{T}$.

Finally, Assumption \ref{contraction} should really be derived from more basic assumptions, and is only stated to allow for a unified analysis. Note that in the optimization case this is a very mild assumption since we are generally only interested in convergent sequences, and for any such sequence with reasonably fast convergence this assumption holds. The following proposition shows that Assumption \ref{contraction} holds for our cases of interest. All the proofs for this section can be found in appendix \ref{appendix2}.

\begin{restatable}{proposition}{contractionsgdsgld}
\label{contractionsgdsgld}
For any choice of $\{p^t\}_{t=1}^{T}$, the iterates of SGD (\ref{SGD}) with the gradient estimator (\ref{estimator}) and decreasing step-sizes $\alpha_t := E/(F + t - 1)^{\beta}$ with $E \geq 0$, $F \geq 1$ and $\beta \in (0,1]$ satisfy Assumption \ref{contraction} with $A := NGE$, $B := F$, and $\delta := \beta$. Under the same conditions, the iterates of SGLD (\ref{SGLD}) satisfy Assumption \ref{contraction} with $A := \sqrt{E}\left(NG\sqrt{\alpha_1} + \sqrt{2d}\right)$, $B := F$, and $\delta := \beta/2$.
\end{restatable}

We now state a proposition that relates the optimal function value for the problem (\ref{optproblem}) over the restricted simplex, with the optimal function value over the whole simplex. Its proof is taken from (\cite{DBLP:conf/colt/Borsos0L18}, Lemma 6):
\begin{restatable}{proposition}{restriction}
\label{restriction}
    Let $\{a_i\}_{i=1}^{N}$ be a non-negative set of numbers, and let $\varepsilon \in [0, 1/2N]$. Then:
    \begin{equation*}
        \label{restrictionbound}
        \min_{p \in \Delta(\varepsilon)} \sum_{i=1}^{N} \frac{1}{p_i}a_i^2 - \min_{p \in \Delta} \sum_{i=1}^{N} \frac{1}{p_i}a_i^2 \leq 6 \varepsilon N \left(\sum_{i=1}^{N} a_i\right)^2
    \end{equation*}
\end{restatable}
The following lemma gives our first bound on the regret per time step:
\begin{restatable}{lemma}{perstepbound}
\label{perstepbound}
Let $q^t := \argmin_{p \in \Delta} \{c_t(p)\}$.
Under Assumption \ref{boundedgradients}, and when using the sequence of distributions defined by (\ref{bestalgever}), we have the following bound for $t \in \{t_0, \dots, T\}$:
\begin{equation*}
    \Exp{c_{t}(p^{t}) - c_{t}(q^t)} \leq \frac{4G}{\varepsilon_{t}} \Exp{\sum_{i=1}^{N}\norm{g_i^t - h_i^t}_2} + 6\varepsilon_{t}G^2N^3
\end{equation*}
where $t_0 := \min\{t \in [T] \suchthat \varepsilon_{t} \leq \frac{1}{2N}\}$.
\end{restatable}
\begin{proofoutline}
Let $\widetilde{p}^{t} := \argmin_{p \in \Delta} \{\widetilde{c}_t(p)\}$. Then we have the following decomposition:
\begin{align*}
    \Exp{c_{t}(p^{t}) - c_{t}(q^t)} &= 
    \underbrace{\Exp{c_t(p^{t}) - \widetilde{c}_t(p^{t})}}_{\text{(A)}}
    + \underbrace{\Exp{\widetilde{c}_t(p^{t}) - \widetilde{c}_t(\widetilde{p}^{t})}}_{\text{(B)}}
    + \underbrace{\Exp{\widetilde{c}_t(\widetilde{p}^{t}) - c_{t}(q^t)}}_{\text{(C)}}
\end{align*}
The terms (A) and (C) are the penalties we pay for using a surrogate cost function, while (B) is the price we pay for restricting the simplex.
Using Assumption \ref{boundedgradients}, proposition \ref{restrictionbound}, and the fact that $p^{t}$ is contained in the $\varepsilon_t$-restricted simplex, each of these terms can be bound to give the result stated.
\end{proofoutline}
The expectation in the first term of the above lemma is highly reminiscent of the first term of the Lyapunov function used to study the convergence of SAGA first proposed in \cite{DBLP:conf/nips/HofmannLLM15} and subsequently refined in \cite{DBLP:conf/nips/Defazio16} (with $g_i^t$ replaced by $g_i^{*}$, the gradient at the minimum). Inspired by this similarity, we prove the following recursion:
\begin{restatable}{lemma}{recursion}
\label{recursion}
Under Assumptions \ref{smoothness} and \ref{contraction}, we have:
\begin{equation*}
    \Exp{\sum_{i=1}^{N}\norm{g_i^{t+1} - h_i^{t+1}}_2} \leq \frac{NLA}{(B + t - 1)^{\delta}}+ (1 - \varepsilon_{t}) \Exp{\sum_{i=1}^{N}\norm{g_i^{t} - h_i^{t}}_2}
\end{equation*}
\end{restatable}
The natural thing to do at this stage is to unroll the above recursion, replace in Lemma \ref{perstepbound}, sum over all time steps, and minimize the obtained regret bound over the choice of the sequence $\{\varepsilon_t\}_{t=1}^{T}$. However, even if we can solve the resulting optimization problem efficiently, the solution will still depend on the constants $G$ and $L$ and on the initial error due to the initialization of the $h_i$s, both of which are usually unknown. Here instead, we make an asymptotic argument to find the optimal decay rate of $\{\varepsilon_t\}_{t=1}^{T}$, propose a sequence that satisfies this decay rate, and show that it gives rise to the dynamic regret bounds stated.

Denote by $\varphi(t)$ the expectation in the first term of the upper bound in Lemma \ref{perstepbound}.
Suppose we take $\varepsilon_t$ to be of order $t^{-\beta}$. Looking at the recursion in Lemma \ref{recursion}, we see that to control the positive term, we need the negative term to be of order at least $t^{-\delta}$, so that $\varphi(t)$ cannot be smaller than $t^{\beta - \delta}$. The bound of Lemma \ref{perstepbound} is therefore of order $t^{2\beta - \delta} + t^{-\beta}$. The minimum is attained when the exponents are equal so we have: $2\beta - \delta = - \beta \implies \beta = \frac{\delta}{3}$

We are now ready to guess the form of $\varepsilon_t$. Matching the positive term in Lemma \ref{recursion}, we consider the following sequence:
\begin{equation}
    \label{epsilonsequence}
    \varepsilon_t := \frac{1}{C^{1 - \delta/3}(C + t -1)^{\delta/3}}
\end{equation}
For a free parameter $C$ satisfying $C \geq N$ to ensure $\varepsilon_1 \leq 1/N$. With this choice of the sequence $\{\varepsilon_t\}_{t=1}^{T}$, we are now finally ready to state our main result of the paper:
\begin{restatable}{theorem}{regretbound}
\label{regretbound}
    Under Assumptions \ref{boundedgradients} and \ref{smoothness} on the functions $f_i$ and Assumption \ref{contraction} on the sequence $\{x_t\}_{t=1}^{T}$, algorithm (\ref{bestalgever}) with the sequence $\{\varepsilon_t\}_{t=1}^{T}$ given in (\ref{epsilonsequence}) satisfies the following dynamic regret bound for all $T \geq t_0$:
    \begin{equation}
        \Exp{\text{Regret}_{D}(T)} \leq \mathcal{O}(T^{1 - \delta/3})
    \end{equation}
    where $t_0 := \min\{t \in [T] \suchthat \varepsilon_{t} \leq \frac{1}{2N}\}$ as in Lemma \ref{perstepbound}.
\end{restatable}
\begin{proofoutline}
Using the sequence given by (\ref{epsilonsequence}) and the recursion in Lemma \ref{recursion}, we show by induction that $\varphi(t) \leq \mathcal{O}(t^{-2\delta/3})$ . Replacing in Lemma \ref{perstepbound}, summing over all time steps, and bounding the resulting sums by the appropriate integrals we get the result.
\end{proofoutline}
Note that since $C \geq N$, we have $t_0 \leq (2^{3/\delta} - 1)N + 1$, so $t_0$ is bounded by a constant independent of $T$. Furthermore, setting $C = 2N$ makes the theorem hold for all $T \in \mathbb{N}$. In practice however, it might be beneficial to set $C = N$ to overcome a bad initialization of the $h_i$s.

Combining Theorem \ref{regretbound} with proposition \ref{contractionsgdsgld} we obtain the following corollary:
\begin{corollary}
\label{regretsgdsgld}
Under Assumptions \ref{boundedgradients} and \ref{smoothness} on the functions $f_i$, if SGD (\ref{SGD}) is run with step-sizes $\mathcal{O}(1/t)$ using the estimator (\ref{estimator}) and probabilities (\ref{bestalgever}) with the sequence $\{\varepsilon_t\}_{t=1}^{T}$ given by (\ref{epsilonsequence}), then for all $T \geq t_0$:
\begin{equation}
    \Exp{\text{Regret}_{D}(T)} \leq \mathcal{O}(T^{2/3})
\end{equation}
and for SGLD (\ref{SGLD}):
\begin{equation}
    \Exp{\text{Regret}_{D}(T)} \leq \mathcal{O}(T^{5/6})
\end{equation}
under the same conditions.
\end{corollary}

\section{A new mini-batch estimator}
\label{sec5}
In most practical applications, one uses a mini-batch of samples to construct the gradient estimator instead of just a single sample. The most basic such estimator is the one formed by sampling a mini-batch of indices $S_t = \{I_{t}^{1}, \dots, I_{t}^{m}\}$ uniformly and independently, and taking the sample mean. This gives an unbiased estimator, whose variance decreases as $1/m$. A simple way to make it more efficient is by sampling the indices uniformly but without replacement. In that case, the variance decreases by an additional factor of $(1-(m-1)/(N-1))$. For $m \ll N$, the difference is negligible, and the additional cost of sampling without replacement is not justified. However, when using unequal probabilities, this argument no longer holds, and the additional variance reduction obtained from sampling without replacement can be significant even for small $m$.

For our problem, besides the additional variance reduction, sampling without replacement allows a higher rate of replacement of the $h_i$s, which is directly related to our regret through the factor in front of the second term of Lemma \ref{recursion}, whose proper generalization for mini-batch sampling is $(1 - \min_{i \in [N]}\pi_i^{t})$, where $\pi_i^t = P(i \in S_t)$ is the inclusion probability of index $i$. This makes sampling without replacement desirable for our purposes. Unfortunately, unlike in the uniform case, the sample mean generalization of (\ref{estimator}) is no longer unbiased. We propose instead the following estimator:
\begin{restatable}{equation}{minibatchestimator}
     \hat{g}^{t}_{b} = \frac{1}{m}\sum_{j=1}^{m} \hat{g}^{t}_{j} \quad\quad
     \hat{g}^{t}_{j} := \left[\frac{1}{q^{t, j}_{I_t^j}} g_{I_t^j}^{t} + \sum_{k=1}^{j-1} g_{I_{t}^k}^{t} \right] \quad\quad 
     q^{t, j}_{i} := \frac{p^{t}_{i}}{1 - \sum_{k=1}^{j-1}p^{t}_{I_{t}^{k}}}
\end{restatable}
We summarize some of its properties in the following proposition:
\begin{restatable}{proposition}{propertiesminibatchestimator}
    \label{propertiesminibatchestimator} 
    Let $S_t^{j} := \{I_t^1, \dots, I_t^{j}\}$ for $j \in [m]$ and $S_t^0 := \emptyset$. We have:
    \begin{enumerate}[\upshape (a)]
        \item $\Exp{\hat{g}_b^t} = g^t$
        \item $\Exp{\norm{\hat{g}_b^t - g^t}_2^2} = (1/m^2)\sum_{j=1}^{m}\Exp{\norm{\hat{g}^{t}_j - g^t}_2^2}$
        \item $\argmin_{p \in \Delta} \{\Exp{\norm{\hat{g}_b^t - g^t}_2^2}\} = \argmin_{p \in \Delta} \{c_t(p)\}$
        \item $\Exp{\norm{\hat{g}^{t}_{j+1} - g^t}_2^2} =  \left(1 - \Exp{q^{t, j}_{I_t^{j}}}\right) \Exp{\norm{\hat{g}^{t}_{j} - g^t}_2^2} - \Exp{q^{t, j}_{I_t^{j}}\norm{\hat{g}^{t}_{j} - g^t}_2^2}$
    \end{enumerate}
    where all the expectations in (d) are conditional on $S_{t}^{j-1}$.
\end{restatable}
The proposed estimator is therefore unbiased and its variance decreases super-linearly in $m$ (by (b) and (d)). Although we were unable to prove a regret bound for this estimator, (c) suggests that it is still reasonable to use our algorithm.
To take into account the higher rate of replacement of the $h_i$s, we propose using the following $\varepsilon_t$ sequence, which is based on the mini-batch equivalent of Lemma \ref{recursion} and the inequality $\min_{i \in [N]}\pi_i^{t} \geq m\varepsilon_t$ \cite{yu2012, MILBRODT1992243, article3}:
\begin{equation}
    \label{minibatchepsilonsequence}
    \varepsilon_t := \frac{1}{C^{1 - \delta/3}(C + m(t-1))^{\delta/3}}
\end{equation}

\section{Experiments}
\label{sec6}
In this section, we present results of experiments with our algorithm. We start by validating our theoretical results through a synthetic experiment. We then show that our proposed method outperforms existing methods on real world datasets. Finally, we identify settings in which adaptive importance sampling can lead to significant performance gains for the final optimization algorithm.

In all experiments, we added $l_2$-regularization to the model's loss and set the regularization parameter $\mu = 1$. We ran SGD (\ref{SGD}) with decreasing step sizes $\alpha_t = \frac{m}{2NL + m \mu t}$ where $m$ is the batch size following \cite{stich2019unified}. We experimented with 3 different samplers in addition to our proposed sampler: Uniform, Multi-armed bandit Sampler (\emph{Mabs}) \cite{DBLP:journals/corr/abs-1708-02544}, and Variance Reducer Bandit (\emph{Vrb}) \citesupp{DBLP:conf/colt/Borsos0L18}. The hyperparameters of both \emph{Mabs} and \emph{Vrb} are set to the ones prescribed by the theory in the original papers. For our propsoed sampler \emph{Avare}, we use the epsilon sequence given by ($\ref{minibatchepsilonsequence}$) with $C = N$, $\delta = 1$, and initialized $h_i = 0$ for all $i \in [N]$. For each sampler, we ran SGD 10 times and averaged the results. The shaded areas represent a one standard deviation confidence interval.

\begin{figure}[b]
  \centering
  \includegraphics[width=5.0in,keepaspectratio]{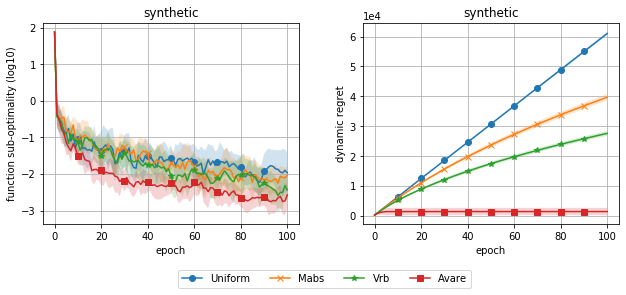}
  \caption{Evolution of function sub-optimality (left) and dynamic regret (right) as a function of data passes on a synthetic dataset with an $l_2$ regularized logistic regression model and different samplers.}
  \label{fig:synthetic}
\end{figure}

To validate our theoretical findings, we randomly generated a dataset for binary classification with $N=100$ and $d=10$.
We then trained a logistic regression model on the generated data, and used a batch size of 1 to match the setting of our regret bound.
The results of this experiment are depicted in figure \ref{fig:synthetic}, and show that \emph{Avare} outperforms the other methods, achieving significantly lower dynamic regret and faster convergence.

For our second experiment, we tested our algorithm on three real world datasets: MNIST, IJCNN1 \cite{10.1145/1961189.1961199}, and CIFAR10. We used a softmax regression model, and a batch size of 128 sampled with replacement. The results of this experiment can be seen in figure \ref{fig:real}. The third column shows the relative error which we define as $\left[c_t(p^t) - \min_{p \in \Delta} c_t(p)\right]/\min_{p \in \Delta} c_t(p)$. In all three cases, \emph{Avare} achieves significantly smaller dynamic regret, with a relative error quickly decaying to zero. The performance gains in the final optimization algorithm are clear in both MNIST and IJCNN1, but are not noticeable in the case of CIFAR10. 

To determine the effectiveness of non-uniform sampling in accelerating the optimization process, previous work \cite{DBLP:conf/icml/ZhaoZ15, DBLP:journals/corr/abs-1708-02544} has suggested to look at the ratio of the maximum smoothness constant and the average smoothness constant. We show here that this is the wrong measure to look at when using adaptive probabilities. In particular, we argue that the ratio of the variance with uniform sampling at the loss minimizer to the optimal variance at the loss minimizer is much more informative of the performance gains achievable through adaptive importance sampling. For each dataset, these ratios are displayed in Table \ref{table:ratios}, supporting our claim. We suspect that for large models capable of fitting the data nearly perfectly, our proposed ratio will be large since many of the per-example gradients at the optimum will be zero. We therefore expect our method to be particularly effective in the training of models of this type. We leave such experiments for future work. Finally, in appendix \ref{appendix4}, we propose an extension of our method to constant step-size SGD, and show that it preserves the performance gains observed when using decreasing step-sizes.

\begin{figure}[t]
  \centering
  \includegraphics[width=5.0in,keepaspectratio]{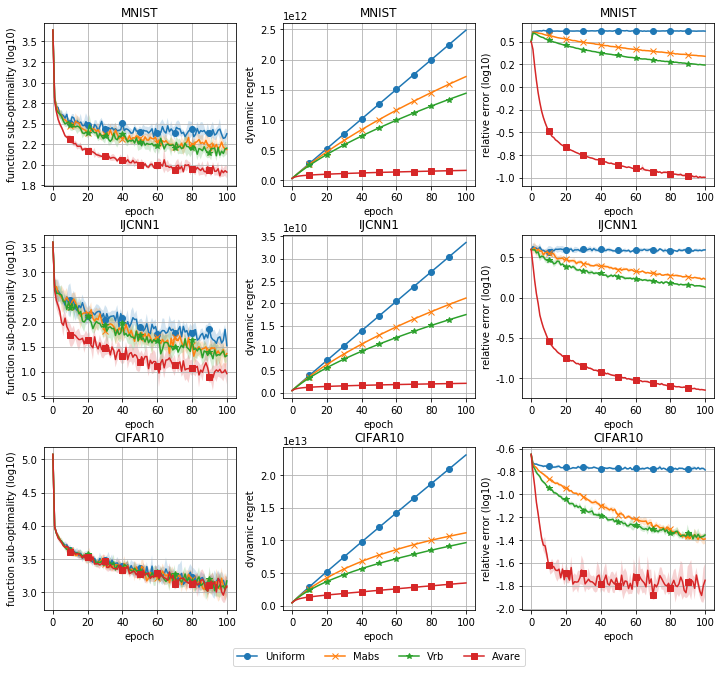}
  \caption{Comparison of the performance of importance samplers on an l2-regularized softmax regression model on three real world datasets: MNIST (top), IJCNN1 (middle), CIFAR10 (bottom).}
  \label{fig:real}
\end{figure}

\begin{table}[t]
\caption{Useful ratios in determining the effectiveness of variance reduction through importance sampling. $L_i$ is the smoothness constant of $f_i$, $L_{max} = \max_{i \in [N]}{L_i}$, and $g_i^{*}$ and is the gradient of $f_i$ at the loss minimizer $x^{*}$.}
\centering
    \begin{tabular}{lcccl}\toprule
        Dataset & $\frac{N L_{max}}{\sum_{i=1}^{N} L_{i}}$ & $\frac{N \sum_{i=1}^{N}\norm{g_i^{*}}_2^2}{(\sum_{i=1}^{N}\norm{g_i^{*}}_2)^2}$ \\\midrule
        Synthetic & 1.69 & 4.46 \\
        MNIST & 3.28 & 5.08 \\
        IJCNN1 & 1.12 & 4.83 \\
        CIFAR10 & 3.40 & 1.14 \\\bottomrule
    \end{tabular}
    \label{table:ratios}
\end{table}

\newpage

\section*{Broader Impact}
Our work develops a new method for variance reduction for stochastic optimization and sampling algorithms.
On the optimization side, we expect our method to be very useful in accelerating the training of large scale neural networks, particularly since our method is expected to provide significant performance gains when the model is able to fit the data nearly perfectly.
On the sampling side, we expect our method to be useful in accelerating the convergence of MCMC algorithms, opening the door for the use of accurate Bayesian methods at a large scale.

\begin{ack}
This research was supported by an NSERC discovery grant. We would like to thank the anonymous reviewers for
their useful comments and suggestions.
\end{ack}

{\footnotesize
\bibliography{mendeley}}

\begin{thebibliography}{10}

\bibitem{robbins1951}
Herbert Robbins and Sutton Monro.
\newblock {A Stochastic Approximation Method}.
\newblock {\em Ann. Math. Statist.}, 22(3):400--407, 1951.

\bibitem{DBLP:conf/icml/WellingT11}
Max Welling and Yee~Whye Teh.
\newblock {Bayesian Learning via Stochastic Gradient Langevin Dynamics}.
\newblock In Lise Getoor and Tobias Scheffer, editors, {\em Proceedings of the
  28th International Conference on Machine Learning, {ICML} 2011, Bellevue,
  Washington, USA, June 28 - July 2, 2011}, pages 681--688. Omnipress, 2011.

\bibitem{DBLP:conf/nips/NeedellWS14}
Deanna Needell, Rachel Ward, and Nathan Srebro.
\newblock {Stochastic Gradient Descent, Weighted Sampling, and the Randomized
  Kaczmarz algorithm}.
\newblock In Zoubin Ghahramani, Max Welling, Corinna Cortes, Neil~D Lawrence,
  and Kilian~Q Weinberger, editors, {\em Advances in Neural Information
  Processing Systems 27: Annual Conference on Neural Information Processing
  Systems 2014, December 8-13 2014, Montreal, Quebec, Canada}, pages
  1017--1025, 2014.

\bibitem{DBLP:conf/icml/ZhaoZ15}
Peilin Zhao and Tong Zhang.
\newblock {Stochastic Optimization with Importance Sampling for Regularized
  Loss Minimization}.
\newblock In Francis~R Bach and David~M Blei, editors, {\em Proceedings of the
  32nd International Conference on Machine Learning, {ICML} 2015, Lille,
  France, 6-11 July 2015}, volume~37 of {\em {JMLR} Workshop and Conference
  Proceedings}, pages 1--9. JMLR.org, 2015.

\bibitem{DBLP:journals/corr/BouchardTPG15}
Guillaume Bouchard, Th{\'{e}}o Trouillon, Julien Perez, and Adrien Gaidon.
\newblock {Accelerating Stochastic Gradient Descent via Online Learning to
  Sample}.
\newblock {\em CoRR}, abs/1506.0, 2015.

\bibitem{DBLP:conf/nips/StichRJ17}
Sebastian~U Stich, Anant Raj, and Martin Jaggi.
\newblock {Safe Adaptive Importance Sampling}.
\newblock In Isabelle Guyon, Ulrike von Luxburg, Samy Bengio, Hanna~M Wallach,
  Rob Fergus, S~V~N Vishwanathan, and Roman Garnett, editors, {\em Advances in
  Neural Information Processing Systems 30: Annual Conference on Neural
  Information Processing Systems 2017, 4-9 December 2017, Long Beach, CA,
  {USA}}, pages 4381--4391, 2017.

\bibitem{DBLP:conf/icml/KatharopoulosF18}
Angelos Katharopoulos and Fran{\c{c}}ois Fleuret.
\newblock {Not All Samples Are Created Equal: Deep Learning with Importance
  Sampling}.
\newblock In Jennifer~G Dy and Andreas Krause, editors, {\em Proceedings of the
  35th International Conference on Machine Learning, {ICML} 2018,
  Stockholmsm{\"{a}}ssan, Stockholm, Sweden, July 10-15, 2018}, volume~80 of
  {\em Proceedings of Machine Learning Research}, pages 2530--2539. PMLR, 2018.

\bibitem{DBLP:conf/nips/JohnsonG18}
Tyler~B Johnson and Carlos Guestrin.
\newblock {Training Deep Models Faster with Robust, Approximate Importance
  Sampling}.
\newblock In Samy Bengio, Hanna~M Wallach, Hugo Larochelle, Kristen Grauman,
  Nicol{\`{o}} Cesa-Bianchi, and Roman Garnett, editors, {\em Advances in
  Neural Information Processing Systems 31: Annual Conference on Neural
  Information Processing Systems 2018, NeurIPS 2018, 3-8 December 2018,
  Montr{\'{e}}al, Canada}, pages 7276--7286, 2018.

\bibitem{DBLP:journals/jmlr/CsibaR18}
Dominik Csiba and Peter Richt{\'{a}}rik.
\newblock {Importance Sampling for Minibatches}.
\newblock {\em J. Mach. Learn. Res.}, 19:27:1----27:21, 2018.

\bibitem{DBLP:conf/icml/NamkoongSYD17}
Hongseok Namkoong, Aman Sinha, Steve Yadlowsky, and John~C Duchi.
\newblock {Adaptive Sampling Probabilities for Non-Smooth Optimization}.
\newblock In Doina Precup and Yee~Whye Teh, editors, {\em Proceedings of the
  34th International Conference on Machine Learning, {ICML} 2017, Sydney, NSW,
  Australia, 6-11 August 2017}, volume~70 of {\em Proceedings of Machine
  Learning Research}, pages 2574--2583. PMLR, 2017.

\bibitem{DBLP:journals/corr/abs-1708-02544}
Farnood Salehi, L~Elisa Celis, and Patrick Thiran.
\newblock {Stochastic Optimization with Bandit Sampling}.
\newblock {\em CoRR}, abs/1708.0, 2017.

\bibitem{DBLP:conf/colt/Borsos0L18}
Zalan Borsos, Andreas Krause, and Kfir~Y Levy.
\newblock {Online Variance Reduction for Stochastic Optimization}.
\newblock In S{\'{e}}bastien Bubeck, Vianney Perchet, and Philippe Rigollet,
  editors, {\em Conference On Learning Theory, {COLT} 2018, Stockholm, Sweden,
  6-9 July 2018}, volume~75 of {\em Proceedings of Machine Learning Research},
  pages 324--357. PMLR, 2018.

\bibitem{DBLP:conf/icml/BorsosCL019}
Zal{\'{a}}n Borsos, Sebastian Curi, Kfir~Yehuda Levy, and Andreas Krause.
\newblock {Online Variance Reduction with Mixtures}.
\newblock In Kamalika Chaudhuri and Ruslan Salakhutdinov, editors, {\em
  Proceedings of the 36th International Conference on Machine Learning, {ICML}
  2019, 9-15 June 2019, Long Beach, California, {USA}}, volume~97 of {\em
  Proceedings of Machine Learning Research}, pages 705--714. PMLR, 2019.

\bibitem{DBLP:journals/ior/BesbesGZ15}
Omar Besbes, Yonatan Gur, and Assaf~J Zeevi.
\newblock {Non-Stationary Stochastic Optimization}.
\newblock {\em Oper. Res.}, 63(5):1227--1244, 2015.

\bibitem{DBLP:journals/corr/abs-1708-03020}
Xi~Chen, Yining Wang, and Yu-Xiang Wang.
\newblock {Non-stationary Stochastic Optimization with Local Spatial and
  Temporal Changes}.
\newblock {\em CoRR}, abs/1708.0, 2017.

\bibitem{DBLP:conf/cdc/MokhtariSJR16}
Aryan Mokhtari, Shahin Shahrampour, Ali Jadbabaie, and Alejandro Ribeiro.
\newblock {Online optimization in dynamic environments: Improved regret rates
  for strongly convex problems}.
\newblock In {\em 55th {IEEE} Conference on Decision and Control, {CDC} 2016,
  Las Vegas, NV, USA, December 12-14, 2016}, pages 7195--7201. IEEE, 2016.

\bibitem{DBLP:conf/icml/YangZJY16}
Tianbao Yang, Lijun Zhang, Rong Jin, and Jinfeng Yi.
\newblock {Tracking Slowly Moving Clairvoyant: Optimal Dynamic Regret of Online
  Learning with True and Noisy Gradient}.
\newblock In Maria-Florina Balcan and Kilian~Q Weinberger, editors, {\em
  Proceedings of the 33nd International Conference on Machine Learning, {ICML}
  2016, New York City, NY, USA, June 19-24, 2016}, volume~48 of {\em {JMLR}
  Workshop and Conference Proceedings}, pages 449--457. JMLR.org, 2016.

\bibitem{DBLP:conf/nips/ZhangLZ18}
Lijun Zhang, Shiyin Lu, and Zhi-Hua Zhou.
\newblock {Adaptive Online Learning in Dynamic Environments}.
\newblock In Samy Bengio, Hanna~M Wallach, Hugo Larochelle, Kristen Grauman,
  Nicol{\`{o}} Cesa-Bianchi, and Roman Garnett, editors, {\em Advances in
  Neural Information Processing Systems 31: Annual Conference on Neural
  Information Processing Systems 2018, NeurIPS 2018, 3-8 December 2018,
  Montr{\'{e}}al, Canada}, pages 1330--1340, 2018.

\bibitem{Beck2003}
Amir Beck and Marc Teboulle.
\newblock {Mirror descent and nonlinear projected subgradient methods for
  convex optimization}.
\newblock {\em Operations Research Letters}, 2003.

\bibitem{Nemirovsky1984}
A.~S. Nemirovsky and D.~B. Yudin.
\newblock {Problem Complexity and Method Efficiency in Optimization}.
\newblock {\em The Journal of the Operational Research Society}, 1984.

\bibitem{DBLP:journals/siamcomp/AuerCFS02}
Peter Auer, Nicol{\`{o}} Cesa-Bianchi, Yoav Freund, and Robert~E Schapire.
\newblock {The Nonstochastic Multiarmed Bandit Problem}.
\newblock {\em {SIAM} J. Comput.}, 32(1):48--77, 2002.

\bibitem{DBLP:conf/icml/Zinkevich03}
Martin Zinkevich.
\newblock {Online Convex Programming and Generalized Infinitesimal Gradient
  Ascent}.
\newblock In Tom Fawcett and Nina Mishra, editors, {\em Machine Learning,
  Proceedings of the Twentieth International Conference {(ICML} 2003), August
  21-24, 2003, Washington, DC, {USA}}, pages 928--936. {AAAI} Press, 2003.

\bibitem{DBLP:journals/ftml/Shalev-Shwartz12}
Shai Shalev-Shwartz.
\newblock {Online Learning and Online Convex Optimization}.
\newblock {\em Foundations and Trends in Machine Learning}, 4(2):107--194,
  2012.

\bibitem{DBLP:journals/ftopt/Hazan16}
Elad Hazan.
\newblock {Introduction to Online Convex Optimization}.
\newblock {\em Found. Trends Optim.}, 2(3-4):157--325, 2016.

\bibitem{HERBSTER1995286}
Mark Herbster and Manfred Warmuth.
\newblock {Tracking the Best Expert}.
\newblock In Armand Prieditis and Stuart Russell, editors, {\em Machine
  Learning Proceedings 1995}, pages 286--294. Morgan Kaufmann, San Francisco
  (CA), 1995.

\bibitem{10.1162/153244301753683726}
Mark Herbster and Manfred~K Warmuth.
\newblock {Tracking the Best Linear Predictor}.
\newblock {\em J. Mach. Learn. Res.}, 1:281--309, sep 2001.

\bibitem{DBLP:conf/nips/Cesa-BianchiGLS12}
Nicol{\`{o}} Cesa-Bianchi, Pierre Gaillard, G{\'{a}}bor Lugosi, and Gilles
  Stoltz.
\newblock {Mirror Descent Meets Fixed Share (and feels no regret)}.
\newblock In Peter~L Bartlett, Fernando C~N Pereira, Christopher J~C Burges,
  L{\'{e}}on Bottou, and Kilian~Q Weinberger, editors, {\em Advances in Neural
  Information Processing Systems 25: 26th Annual Conference on Neural
  Information Processing Systems 2012. Proceedings of a meeting held December
  3-6, 2012, Lake Tahoe, Nevada, United States}, pages 989--997, 2012.

\bibitem{DBLP:conf/icml/GyorgyS16}
Andr{\'{a}}s Gy{\"{o}}rgy and Csaba Szepesv{\'{a}}ri.
\newblock {Shifting Regret, Mirror Descent, and Matrices}.
\newblock In Maria-Florina Balcan and Kilian~Q Weinberger, editors, {\em
  Proceedings of the 33nd International Conference on Machine Learning, {ICML}
  2016, New York City, NY, USA, June 19-24, 2016}, volume~48 of {\em {JMLR}
  Workshop and Conference Proceedings}, pages 2943--2951. JMLR.org, 2016.

\bibitem{DBLP:journals/corr/DefazioBL14}
Aaron Defazio, Francis~R. Bach, and Simon Lacoste{-}Julien.
\newblock {SAGA:} {A} fast incremental gradient method with support for
  non-strongly convex composite objectives.
\newblock {\em CoRR}, abs/1407.0202, 2014.

\bibitem{DBLP:conf/nips/HofmannLLM15}
Thomas Hofmann, Aur{\'{e}}lien Lucchi, Simon Lacoste-Julien, and Brian
  McWilliams.
\newblock {Variance Reduced Stochastic Gradient Descent with Neighbors}.
\newblock In Corinna Cortes, Neil~D Lawrence, Daniel~D Lee, Masashi Sugiyama,
  and Roman Garnett, editors, {\em Advances in Neural Information Processing
  Systems 28: Annual Conference on Neural Information Processing Systems 2015,
  December 7-12, 2015, Montreal, Quebec, Canada}, pages 2305--2313, 2015.

\bibitem{DBLP:conf/nips/Defazio16}
Aaron Defazio.
\newblock {A Simple Practical Accelerated Method for Finite Sums}.
\newblock In Daniel~D Lee, Masashi Sugiyama, Ulrike von Luxburg, Isabelle
  Guyon, and Roman Garnett, editors, {\em Advances in Neural Information
  Processing Systems 29: Annual Conference on Neural Information Processing
  Systems 2016, December 5-10, 2016, Barcelona, Spain}, pages 676--684, 2016.

\bibitem{yu2012}
Yaming Yu.
\newblock {On the inclusion probabilities in some unequal probability sampling
  plans without replacement}.
\newblock {\em Bernoulli}, 18(1):279--289, 2012.

\bibitem{MILBRODT1992243}
Hartmut Milbrodt.
\newblock {Comparing inclusion probabilities and drawing probabilities for
  rejective sampling and successive sampling}.
\newblock {\em Statistics and Probability Letters}, 14(3):243--246, 1992.

\bibitem{article3}
Talluri Rao, Samindranath Sengupta, and B~Sinha.
\newblock {Some order relations between selection and inclusion probabilities
  for PPSWOR sampling scheme}.
\newblock {\em Metrika}, 38:335--343, 1991.

\bibitem{stich2019unified}
Sebastian~U. Stich.
\newblock Unified optimal analysis of the (stochastic) gradient method, 2019.

\bibitem{10.1145/1961189.1961199}
Chih-Chung Chang and Chih-Jen Lin.
\newblock Libsvm: A library for support vector machines.
\newblock {\em ACM Trans. Intell. Syst. Technol.}, 2(3), May 2011.

\end{thebibliography}


\begin{thebibliography}{1}

\bibitem{DBLP:conf/colt/Borsos0L18}
Zalan Borsos, Andreas Krause, and Kfir~Y Levy.
\newblock {Online Variance Reduction for Stochastic Optimization}.
\newblock In S{\'{e}}bastien Bubeck, Vianney Perchet, and Philippe Rigollet,
  editors, {\em Conference On Learning Theory, {COLT} 2018, Stockholm, Sweden,
  6-9 July 2018}, volume~75 of {\em Proceedings of Machine Learning Research},
  pages 324--357. PMLR, 2018.

\bibitem{DBLP:books/daglib/0023376}
Thomas~H. Cormen, Charles~E. Leiserson, Ronald~L. Rivest, and Clifford Stein.
\newblock {\em Introduction to Algorithms, 3rd Edition}.
\newblock {MIT} Press, 2009.

\end{thebibliography}
\bibliographystyle{unsrt}

\newpage

\begin{appendices}
    \section{Algorithm}
    \label{appendix1}
    \subsection{Proof of Lemma \ref{solutionoptproblem}}
    \solutionoptproblem*
    \begin{proof}
    ~
    \paragraph{Edge case and well-definedness.}
    If $a_i = 0$ for all $i \in [N]$, then the objective function is identically $0$ over $\Delta(\varepsilon)$ for all $\varepsilon \in [0, 1/N]$, so that any $p \in \Delta(\varepsilon)$ is a solution (we set $(1/0)0 := 0$ in the objective). Else there exists an $i \in [N]$ such that $a_i > 0$, and therefore $a_{\pi(1)} > 0$. Now:
    \begin{align*}
        &\frac{\varepsilon}{1-(N-1)\varepsilon} \leq 1 \\
        &\Leftrightarrow \varepsilon \leq 1 - (N-1)\varepsilon \\
        &\Leftrightarrow 0 \leq 1 - N\varepsilon \\
        &\Leftrightarrow N\varepsilon \leq 1\\
        &\Leftrightarrow \varepsilon \leq \frac{1}{N}
    \end{align*}
    The last inequality is true, so it implies the first, and we have:
    \begin{equation*}
        a_{\pi(1)} \geq \varepsilon\frac{a_{\pi(1)}}{1-(N-1)\varepsilon}
    \end{equation*}
    Therefore $\rho$ is well defined and is $\geq 1$. As a consequence, $\lambda$ is also well defined and is $> 0$, making $p^{*}$ in turn well-defined.
    
    \paragraph{Optimality proof.}It is easily verified that problem (\ref{optproblem}) is convex. Its Lagrangian is given by:
    \begin{equation*}
        \mathcal{L}(p, \mu, \nu) = \sum_{i=1}^{N} \frac{1}{p_i}a_i^2 - \sum_{i=1}^{N}\mu_i (p_i - \varepsilon) + \nu \left(\sum_{i=1}^{N}p_i - 1\right)
    \end{equation*}
    and the KKT conditions are:
    \begin{align*}
        \text{(Stationarity)} &\quad p_i = \frac{a_i}{\sqrt{\nu - \mu_i}} \\
        \text{(Complementary slackness)} &\quad \mu_i = 0 \thinspace \lor \thinspace p_i = \varepsilon \\
        \text{(Primal feasibility)} &\quad p_i \geq \varepsilon \thinspace \land \thinspace \sum_{j=1}^{N} p_j = 1 \\
        \text{(Dual feasibility)} &\quad \mu_i \geq 0
    \end{align*}
    
    By convexity of the problem, the KKT conditions are sufficient conditions for global optimality. To show that our proposed solution is optimal, it therefore suffices to exhibit constants $(\mu^{*}_i)_{i=1}^{N}$ and $\nu^{*}$ that together with $p^{*}$ satisfy these conditions.
    Let:
    \begin{align*}
        \nu^{*} &:= \lambda^2 \\
        \mu^{*}_i &:= \begin{cases}
        0 &\text{if $i \in \{\pi(1), \dots, \pi(\rho)\}$}\\
        \nu^{*} - a_i^{2}/\varepsilon^2 &\text{otherwise}
        \end{cases}
    \end{align*}
    
    Note that the $\mu_i^{*}$s are well defined since when $\varepsilon = 0$, $\rho = N$, and $\mu_i^{*} = 0$ for all $i \in [N]$. We claim that the triplet $(p^{*}, (\mu^{*}_i)_{i=1}^{N}, \nu^{*})$ satisfies the KKT conditions.
    
    Stationarity and complementary slackness are immediate from the definitions. The first clause of primal feasibility holds by definition of $p^{*}_i$ for $i \in \{\pi(\rho + 1), \dots, \pi(N)\}$. In the other case we have:
    \begin{align*}
        p_i^{*} &= \frac{a_i}{\lambda} \\ 
        &= \frac{a_i}{\sum_{j=1}^{\rho}a_{\pi(j)}} \left(1 - (N-\rho)\varepsilon\right) \\
        &\geq \frac{a_{\pi(\rho)}}{\sum_{j=1}^{\rho}a_{\pi(j)}} \left(1 - (N-\rho)\varepsilon\right) \\
        &\geq \varepsilon
    \end{align*}
    where in the third line we used the fact that $\pi$ orders $\{a_i\}_{i=1}^{N}$ in decreasing order, and in the last line we used the inequality in the definition of $\rho$. For the second clause of primal feasibility:
    \begin{equation*}
        \sum_{j=1}^{N} p^{*}_{j} = \sum_{j=1}^{N} p^{*}_{\pi(j)} = \sum_{j=1}^{\rho} \frac{a_{\pi(i)}}{\lambda} + \sum_{j=\rho+1}^{N} \varepsilon = \left(1 - (N - \rho) \varepsilon\right) + (N - \rho) \varepsilon = 1
    \end{equation*}
    Finally, dual feasibility holds by definition of $\mu^{*}_i$ for $i \in \{\pi(1), \dots, \pi(\rho)\}$. In the other case we have:
    \begin{align*}
        \mu_i^{*} &= \nu^{*} - \frac{a_i^2}{\varepsilon^2} \\
        &= \lambda^2 - \frac{a_i^2}{\varepsilon^2} \\
        &= \left(\lambda + \frac{a_i}{\varepsilon}\right) \left(\lambda - \frac{a_i}{\varepsilon}\right) \\
        & \geq \left(\lambda + \frac{a_{i}}{\varepsilon}\right)\left(\lambda - \frac{a_{\pi(\rho+1)}}{\varepsilon}\right)
    \end{align*}
    The first factor is positive by positivity of $\lambda$ and non-negativity of the $a_i$s. For the second factor, we have by the maximality of $\rho$:
    \begin{align}
    &a_{\pi(p+1)} < \varepsilon \frac{\sum_{j=1}^{\rho+1}a_{\pi(j)}}{1 - (N - \rho - 1) \varepsilon} \nonumber \\
    &\Rightarrow a_{\pi(\rho + 1)}(1 - (N - \rho) \varepsilon + \varepsilon) < \varepsilon \sum_{i=1}^{\rho}a_{\pi(j)} + \varepsilon a_{\pi(\rho+1)} \nonumber \\
    &\Rightarrow a_{\pi(\rho+1)}(1 - (N-\rho)\varepsilon) < \varepsilon \sum_{i=1}^{\rho}a_{\pi(j)} \nonumber \\
    &\Rightarrow \frac{a_{\pi(\rho+1)}}{\varepsilon} < \lambda
    \label{usefulineq}
    \end{align}
    Therefore the second factor is also positive, and dual feasibility holds.
    \end{proof}
    
    \subsection{Implementation and complexity}
    In this section of the appendix, we provide pseudocode for the implementation of the algorithm and discuss its computational complexity. 
    
    \begin{algorithm}
    \linespread{1.35}\selectfont
    \caption{Implementation of the proposed sampler}
    \label{pseudocodesampler}
    \begin{algorithmic}[1]
    
    \Class{Sampler}:       
        \Procedure{Initialize}{$\{\norm{h_i}_2\}_{i=1}^{N}$}
        \State $self.H \gets Array(\{\norm{h_i}_2\}_{i=1}^{N})$
        \State $self.T$ $\gets$ $OST(keys=\{\norm{h_i}_2\}_{i=1}^{N}, values=[N])$
        \State $self.CS \gets Array(\{\sum_{j=1}^{i}\norm{h_{\pi(j)}}_2\}_{i=1}^{N})$
        \EndProcedure
        
        \vspace{1ex}
        
        \Procedure{Delete}{$x$}
        \State $r$ $\gets$ $self.T.rank(x)$
        \State $self.T.delete(x)$
        \State $self.CS[r:N]$ $\gets$ $self.CS[r:N] - x$
        \EndProcedure
        
        \vspace{1ex}
        
        \Procedure{Insert}{$x$, $i$}
        \State $self.T.insert(key=x, value=i)$
        \State $r$ $\gets$ $self.T.rank(x)$
        \State $self.CS[r:N]$ $\gets$ $self.CS[r:N] + x$
        \EndProcedure
        
        \vspace{1ex}
        
        \Procedure{Update}{$\norm{h_{I}}_2$, $I$}
        \State $self.delete(self.H[I])$
        \State $self.H[I] \gets \norm{h_{I}}_2$
        \State $self.insert(\norm{h_{I}}_2, I)$
        \EndProcedure
        
        \vspace{1ex}
        
        \Function{Search}{$\varepsilon$, node}
        \State $r \gets self.T.rank(node)$
        \IfThen{$r == N$}{\Return $r$}
        \State $c \gets 1 - (N - r)\varepsilon$
        \If{$ c \cdot node.key < \varepsilon \cdot self.CS[r]$}
            \State \Return $self.search(\varepsilon, node.left)$
        \Else
            \State $d = 1 - (N - r - 1)\varepsilon$
            \If{$d \cdot node.successor.key$ < $\varepsilon \cdot self.CS[r+1]$}
                \State \Return $r$
            \Else
                \State \Return $self.search(\varepsilon, node.right)$
            \EndIf
        \EndIf
        \EndFunction
        
        \vspace{1ex}
        
        \Function{Sample}{$\varepsilon$}
        \State $\rho \gets self.search(\varepsilon, self.T.root)$
        \State $\lambda \gets self.CS[\rho]/\left(1 - (N - \rho)\varepsilon\right)$
        \State $b \sim Bernoulli((N-\rho)\varepsilon)$
        \If{b == 1}
            \State Sample an index $\tilde{I}$ uniformly from $\{\rho+1, \dots, N\}$.
        \Else
            \State $u \sim Uniform([0, 1])$
            \State Find the first index $\tilde{I}$ for which $\lambda u \leq self.CS[\tilde{I}]$ using binary search.
        \EndIf
        \State $I \gets self.T.select(\tilde{I})$
        \State \Return $I$
        \EndFunction
    \EndClass
    \end{algorithmic}
    \end{algorithm}
    
    \begin{algorithm}
    \linespread{1.35}\selectfont
    \caption{AVARE}
    \label{ISSGD/ISSGLD}
    \hspace*{\algorithmicindent}\textbf{Input}: $x_1, T, \{\alpha_t\}_{t=1}^{T}, \{\norm{h_i^1}_2\}_{i=1}^{N}, C \geq N$
    \begin{algorithmic}[1]
        \State $sampler \gets$ \textsc{Sampler}$(\{\norm{h_i^1}_2\}_{i=1}^{N})$
        \For{$t = 1, 2, \dots, T$}
        \State Set $\varepsilon_t$ according to (\ref{epsilonsequence})
        \State $I_t \gets sampler.sample(\varepsilon_t)$
        \State Obtain $x_{t+1}$ using (\ref{SGD}) or (\ref{SGLD}) and the estimator (\ref{estimator}).
        \State $sampler.update(\norm{g_{I_t}}_2, I_t)$
        \EndFor
        \State \Return $x_{T+1}$
    \end{algorithmic}
    \end{algorithm}
    
    \subsubsection{Implementation}
    We assume in algorithm \ref{pseudocodesampler} that an order statistic tree (OST) (see, for example, \citesupp{DBLP:books/daglib/0023376}, chapter 14) can be instantiated and that the ordering in the tree is such that the key of the left child of a node is greater than or equal to that of the node itself. Furthermore, we assume that the $rank(x)$ method returns the position of $x$ in the order determined by an inorder traversal of the tree. Finally, the $select(i)$ method returns the value of the $i^{th}$-largest key in the tree.
    
    The algorithm works as follows. At initialization, three data structures are initialized: an arrray $H$ holding the gradient norms according to the original indices, an order statistic tree $T$ holding the gradient norms as keys and the original indices as values, and an array $CS$ holding the cumulative sums of the gradient norms, where the sums are accumulated in the (decreasing) order that sorts the gradients norms in $T$.
    
    The $sample(\varepsilon)$ method allows to sample from the optimal distribution on $\Delta(\varepsilon)$. It uses the $search(\varepsilon, node)$ method to find $\rho$ by searching the tree $T$ and using the maximality property of $\rho$. Once $\rho$ is determined, $\lambda$ can be calculated. Using the fact that the cumulative sums are proportional to the CDF of the distribution, the algorithm then samples an index using inverse-transform sampling. The sampled index is then transformed back to an index in the original order using the $select$ method of the tree $T$.
    
    Finally, the $update(\norm{h_{I}}_2, I)$ method replaces the gradient norm of a given index by a new one. It calls the methods $delete(x)$ and $insert(x, i)$ which perform the deletion and insertion while maintaining the tree $T$ and array $CS$.
    
    \subsubsection{Complexity}
    First, let us analyze the cost of running the $update(\norm{h_{I}}_2, I)$ method. For the array $H$, we only use random access and assignment, which are both $O(1)$. For the tree $T$, we use the methods $insert$, $delete$, and $rank$, all of which are $O(\log N)$. Finally, for the array $CS$ we add and subtract from a sub-array, which takes $O(N)$ time, although this operation is vectorized and very fast in practice.
    
    Let us now look at the cost of running $sample(\varepsilon)$. The $search$ method is recursive, but will only be called at most as many times as the height of the tree, which is $O(\log N)$. Now for each call of $search$, both the $rank$ and $successor$ methods of the tree $T$ require $O(\log N)$ time. The rest of the $search$ method only requires $O(1)$ operations. The total cost of the $search$ method is therefore $O(\log^2 N)$. For the rest of the $sample$ method, the operations that dominate the cost are the $select$ method of the tree $T$, which takes $O(\log N)$ time, and the binary search in the else branch, which also runs in $O(\log N)$ time. Consequently, the total cost of the sample method is $O(\log^2 N)$.
    
    The total per iteration cost of using the proposed sampler is therefore $O(N)$ vectorized operations, and $O(\log^2 N)$ sequential (non-vectorized) operations. The total memory cost is $O(N)$.
    
    \section{Theory}
    \label{appendix2}
    We restate our assumptions here for ease of reference.
    \boundedgradients*
    \smoothness*
    \contraction*
    \subsection{Proof of proposition \ref{contractionsgdsgld}}
    \contractionsgdsgld*
    \begin{proof}
    Conditioning on the knowledge of $\{I_{1}, \dots, I_{t-1}\}$ we have for SGD:
    \begin{align*}
        \Exp{\norm{x_{t+1}^{SGD}-x_{t}^{SGD}}_2} &= \Exp{\alpha_t \frac{1}{p_{I_t}^{t}} \norm{g_{I_t}^t}_2} \\
        &= \alpha_t \sum_{i=1}^{N} \norm{g_i^t}_{2} \\
        &\leq \alpha_t NG \\
    \end{align*}
    and for SGLD we have:
    \begin{align*}
        \Exp{\norm{x_{t+1}^{SGLD}-x_{t}^{SGLD}}_2} &\leq \Exp{\alpha_t\frac{1}{p_{I_t}^{t}}\norm{g_{I_t}^{t}}_2} + \Exp{\norm{\xi_t}_2} \\
        &\leq \alpha_t \sum_{i=1}^{N} \norm{g_i^t}_{2} + \sqrt{\Exp{\norm{\xi_t}_2^2}} \\
        &\leq \alpha_t NG + \sqrt{\alpha_t}\sqrt{2d} \\
        &\leq \sqrt{\alpha_t} \left(NG\sqrt{\alpha_1} + \sqrt{2d}\right)
    \end{align*}
    where in the first line we used the triangle inequality, in the second we used Jensen's inequality, and in the last we used the fact that $\{\alpha_t\}_{t=1}^{T}$ is decreasing. Replacing with the value of $\alpha_t$ we obtain the result.
    \end{proof}
    
    \subsection{Proof of proposition \ref{restriction}}
    The following proof is taken from (\cite{DBLP:conf/colt/Borsos0L18}, Lemma 6). 
    \restriction*
    \begin{proof}
        By Lemma \ref{solutionoptproblem} we have:
        \begin{align*}
            \min_{p \in \Delta(\varepsilon)}\sum_{i=1}^{N}\frac{1}{p_i}a_i^2  &= \lambda \sum_{i=1}^{\rho} a_{\pi(i)} + \sum_{i=\rho+1}^{N} \frac{a^{2}_{\pi(i)}}{\varepsilon} \\
            &\leq \lambda^2 \left(1 - (N - \rho)\varepsilon\right) + \varepsilon \sum_{i=\rho+1}^{N} \frac{a^2_{\pi(\rho+1)}}{\varepsilon^2} \\
            &\leq \lambda^2 \left(1 - (N - \rho)\varepsilon\right) + (N-\rho)\varepsilon\lambda^2 \\
            &= \lambda^2
        \end{align*}
        where in the third line we used inequality (\ref{usefulineq}) from the proof of Lemma \ref{solutionoptproblem}.
        Now for the case $\varepsilon = 0$ we have $\rho = N$, so the second term in the first line is zero and the inequality becomes an equality:
        \begin{equation}
            \label{optfunctionvalue}
            \min_{p \in \Delta} \sum_{i=1}^{N}\frac{1}{p_i}a_i^2 = \left(\sum_{i=1}^{N}a_i\right)^2
        \end{equation}
        The difference is therefore bounded by:
        \begin{align*}
            \min_{p \in \Delta(\varepsilon)}\sum_{i=1}^{N}\frac{1}{p_i}a_i^2 - \min_{p \in \Delta}\sum_{i=1}^{N}\frac{1}{p_i}a_i^2 &\leq \frac{\left(\sum_{i=1}^{\rho}a_{\pi(i)}\right)^2}{(1 - (N - \rho)\varepsilon)^2} - \left(\sum_{i=1}^{N} a_i\right)^2 \\
            &\leq \left(\frac{1}{(1-N\varepsilon)^2} - 1\right) \left(\sum_{i=1}^{N} a_i\right)^2 \\
            &\leq 6 \varepsilon N \left(\sum_{i=1}^{N} a_i\right)^2
        \end{align*}
        where in the last line we used the inequality $\frac{1}{(1-x)^2} - 1 \leq 6x$ for $x \in [0, 1/2]$ which gives the restriction $\varepsilon \in [0, 1/2N]$.
    \end{proof}    
    
    \subsection{Proof of Lemma \ref{perstepbound}}
    \perstepbound*
    \begin{proof}
        Let $t \geq t_0$ and $\widetilde{p}^{t} := \argmin_{p \in \Delta} \{\widetilde{c}_t(p)\}$. We have the following decomposition:
        \begin{equation*}
        \Exp{c_{t}(p^{t}) - c_{t}(q^t)} = 
        \underbrace{\Exp{c_t(p^{t}) - \widetilde{c}_t(p^{t})}}_{\text{(A)}}
        + \underbrace{\Exp{\widetilde{c}_t(p^{t}) - \widetilde{c}_t(\widetilde{p}^{t})}}_{\text{(B)}}
        + \underbrace{\Exp{\widetilde{c}_t(\widetilde{p}^{t}) - c_{t}(q^t)}}_{\text{(C)}}
        \end{equation*}
        We bound each term separately:
        \begin{align*}
            \text{(A)} &= \Exp{\sum_{i=1}^{N}\frac{1}{p_i^{t}}\left(\norm{g_i^t}_2^2 - \norm{h_i^t}_2^2\right)} \\
            &= \Exp{\sum_{i=1}^{N}\frac{1}{p_i^t}\left(\norm{g_i^t}_2 - \norm{h_i^t}_2\right)\left(\norm{g_i^t}_2 + \norm{h_i^t}_2\right)} \\
            &\leq \Exp{\sum_{i=1}^{N}\frac{2G}{p_i^t}\norm{g_i^t - h_i^t}_2} \\
            &\leq \frac{2G}{\varepsilon_t}\Exp{\sum_{i=1}^{N}\norm{g_i^t - h_i^t}_2}
        \end{align*}
        where in the third line we used Assumption \ref{boundedgradients} and the reverse triangle inequality, and in the last line we used the fact that $p^t \in \Delta(\varepsilon_t)$.
        Since $t \geq t_0$, we can apply proposition \ref{restriction} on (B) to obtain:
        \begin{align*}
            \text{(B)} &\leq \Exp{6\varepsilon_t N \left(\sum_{i=1}^{N}\norm{h_i^t}_2\right)^2} \leq 6 \varepsilon_t G^2 N^3
        \end{align*}
        where the second inequality uses Assumption \ref{boundedgradients}.
        Finally, using the optimal function values (\ref{optfunctionvalue}) we have for (C):
        \begin{align*}
            \text{(C)} &= \Exp{\left(\sum_{i=1}^{N}\norm{h_i^t}_2\right)^2 - \left(\sum_{i=1}^{N}\norm{g_i^t}_2\right)^2} \\
            &= \Exp{\left(\sum_{i=1}^{N}\left(\norm{h_i^t}_2 - \norm{g_i^t}_2\right) \right) \left(\sum_{i=1}^{N}(\norm{g_i^t}_2 + \norm{h_i^t}_2)\right)} \\
            &\leq 2GN \Exp{\sum_{i=1}^{N}\norm{g_i^t - h_i^t}_2} \\
            &\leq \frac{2G}{\varepsilon_t}\Exp{\sum_{i=1}^{N}\norm{g_i^t - h_i^t}_2}
        \end{align*}
        where we again used Assumption \ref{boundedgradients} and the reverse triangle inequality in the third line. the last inequality follows from the fact that $\varepsilon_t \leq \frac{1}{N}$. Combining the three bounds gives the result.
    \end{proof}
    
    \subsection{Proof of Lemma \ref{recursion}}
    \recursion*
    \begin{proof}
        Conditioning on the knowledge of $\{I_1, \dots, I_{t-1}\}$ we have:
        \begin{align*}
            \Exp{\sum_{i=1}^{N}\norm{g_i^{t+1} - h_i^{t+1}}_2} &= \sum_{j=1}^{N} p_j^{t} \left[\norm{g_j^{t+1} - g_j^{t}}_2 + \sum_{\substack{i=1\\ i \ne j}}^{N}\norm{g_i^{t+1} - h_i^{t}}_2\right] \\
            &\leq \sum_{j=1}^{N} p_{j}^{t} \left[\norm{g_j^{t+1} - g_j^{t}}_2 + \sum_{\substack{i=1\\ i \ne j}}^{N}\left(\norm{g_i^{t+1} - g_i^{t}}_2 + \norm{g_i^t - h_i^t}_2\right)\right] \\
            &= \sum_{j=1}^{N} p_{j}^{t} \left[\sum_{i=1}^{N}\norm{g_i^{t+1} - g_i^{t}}_2 + \sum_{\substack{i=1\\ i \ne j}}^{N} \norm{g_i^t - h_i^t}_2 \right] \\
            &\leq NL \sum_{j=1}^{N}p_j^t \norm{x_{t+1} - x_{t}}_2 + \sum_{i=1}^{N} \left(\sum_{\substack{j=1 \\ j \ne i}}^{N} p_j^t\right) \norm{g_i^t - h_i^t}_2 \\
            &= NL \Exp{\norm{x_{t+1} - x_t}_2} + \sum_{i=1}^{N} (1 - p_i^t)\norm{g_i^t - h_i^t}_2 \\
            &\leq \frac{NLA}{(B+t-1)^{\delta}} + (1 - \varepsilon_t) \sum_{i=1}^{N} \norm{g_i^t - h_i^t}_2
        \end{align*}
        where in the fourth line we used Assumption \ref{smoothness}, and in the last line we used Assumption \ref{contraction} and the fact that $p^t \in \Delta(\varepsilon_t)$.
        Taking expectation with respect to the choice of $\{I_1, \dots, I_{t-1}\}$ on both sides we get the result.
    \end{proof}
    
    \subsection{Lemma \ref{solutionrecursion}}
    Before proving Theorem \ref{regretbound}, we first state and prove the following solution of the recursion of Lemma \ref{recursion} assuming the sequence $\{\varepsilon_t\}_{t=1}^{T}$ is given by (\ref{epsilonsequence}).
    \begin{lemma}
    \label{solutionrecursion}
        Assuming the use of the sequence $\{\varepsilon_t\}_{t=1}^{T}$ given by (\ref{epsilonsequence}) we have:
        \begin{equation*}
            \Exp{\sum_{i=1}^{N}\norm{g_i^{t} - h_i^{t}}_2} \leq \frac{K}{(C+t-1)^{2\delta/3}}
        \end{equation*}
        where:
        \begin{equation*}
            K := \max \left\{
            \frac{3C^{1-\delta/3} D}{3-2\delta},  
            C^{2\delta/3} \sum_{i=1}^{N} \norm{g_i^{1} - h_i^{1}}_2
            \right\} \\
        \end{equation*}
        and:
        \begin{equation*}
            D := \begin{cases}
            NLA &\text{if $B \geq C$} \\
            (\frac{C}{B})^{\delta} NLA &\text{if $B < C$}
            \end{cases}
        \end{equation*}
        where $A$, $B$, and $\delta$ are as in Assumption \ref{contraction}.
    \end{lemma}
    \begin{proof}
        ~
        \paragraph{A simple inequality.}
        Suppose $B \geq C$, then:
        \begin{equation*}
            \frac{NLA}{(B + t - 1)^{\delta}} \leq \frac{NLA}{(C + t - 1)^{\delta}}
        \end{equation*}
        otherwise, we have:
        \begin{align*}
            \frac{NLA}{(B + t -1)^{\delta}} = \frac{(\frac{C}{B})^{\delta} NLA}{(C + (\frac{C}{B})(t-1))^{\delta}} \leq \frac{(\frac{C}{B})^{\delta} NLA}{(C + t - 1)^{\delta}}
        \end{align*}
        where the last inequality follows from the fact that $C > B$ and $t \geq 1$. We conclude that:
        \begin{equation*}
            \frac{NLA}{(B + t - 1)^{\delta}} \leq \frac{D}{(C + t - 1)^{\delta}}
        \end{equation*}
        \\
        \paragraph{Induction proof.} Let $\varphi(t) := \Exp{\sum_{i=1}^{N}\norm{g_i^{t} - h_i^{t}}_2}$, and let $K' := C^{2 - 2\delta/3}K$.
        For $t = 1$ the statement holds since:
        \begin{equation*}
            \varphi(1) = \frac{C^{2\delta/3}\varphi(1)}{(C+t-1)^{2\delta/3}} \leq \frac{K}{(C + t -1)^{2\delta/3}}
        \end{equation*}
        For $t \geq 1$ we have by Lemma \ref{recursion} and the above inequality:
        \begin{align*}
            \varphi(t+1) &\leq \frac{D}{(C+t-1)^{\delta}} + \left(1 - \frac{1}{C^{1 - \delta/3}(C + t - 1)^{\delta/3}}\right) \varphi(t) \\
            &= \frac{aC^{3 - \delta}D}{aC^{3-\delta}(C + t -1)^{\delta}} + \left(1 - \frac{1}{C^{1 - \delta/3}(C + t - 1)^{\delta/3}}\right) \varphi(t) \\
            &\leq \frac{K'}{aC^{3-\delta}(C + t -1)^{\delta}} + \left(1 - \frac{1}{C^{1 - \delta/3}(C + t - 1)^{\delta/3}}\right) \frac{K'}{C^{2 - 2\delta/3}(C + t -1)^{2\delta/3}}
        \end{align*}
        where $a = \frac{3}{3-2\delta}$ and where the last line follows by the induction hypothesis. To simplify notation let $x := (C + t -1)$, $E := C^{1-\delta/3}$, $\gamma := (1-\frac{1}{a}) = \frac{2\delta}{3}$. Then the above inequality can be rewritten as:
        \begin{align*}
            \varphi(t+1) &\leq K' \left(\frac{1}{E^2x^{2\delta/3}} - \frac{\gamma}{E^3x^{\delta}}\right) \\
            &= K' \frac{Ex^{\delta/3} - \gamma}{E^3x^{\delta}} \\
            &= K' \frac{E^3x^{\delta} - \gamma^3}{E^3x^{\delta}(E^2x^{2\delta/3} + E\gamma x^{\delta/3} + \gamma^2)} \\
            &\leq K' \frac{1}{E^2x^{2\delta/3} + E\gamma x^{\delta/3}}\\
        \end{align*}
        Now by concavity of $x^{2\delta/3}$ we have:
        \begin{equation*}
            E^2 \left[(x+1)^{2\delta/3} - x^{2\delta/3}\right] \leq E^2\frac{2\delta}{3} x^{2\delta/3 -1}
        \end{equation*}
        so that:
        \begin{align*}
            & E^2x^{2\delta/3} + E\gamma x^{\delta/3} \geq E^2 (x + 1)^{2\delta/3} \\
            &\Leftrightarrow E\gamma x^{\delta/3} \geq E^2 \left[(x + 1)^{2\delta/3} - x^{2\delta/3}\right] \\
            & \Leftarrow \gamma E x^{\delta/3} \geq \frac{2\delta}{3} E^2 x^{2\delta/3 - 1}\\
            & \Leftrightarrow x^{1-\delta/3} \geq C^{1-\delta/3} \\
            & \Leftrightarrow x \geq C \\
            & \Leftrightarrow (C + t - 1) \geq C \\
            & \Leftrightarrow t \geq 1
        \end{align*}
        The last statement is true, and therefore so is the first. Replacing in the bound on $\varphi(t+1)$ we get:
        \begin{equation*}
            \varphi(t+1) \leq \frac{K'}{C^{2-2\delta/3}(C + (t + 1) - 1)^{2\delta/3}} = \frac{K}{(C + (t + 1) - 1)^{2\delta/3}}
        \end{equation*}
        which finishes the proof.
    \end{proof}
    
    \subsection{Proof of Theorem \ref{regretbound}}
    \regretbound*
    \begin{proof}
        Combining Lemma \ref{solutionrecursion} with Lemma \ref{perstepbound} we have the following per-step bound for $t \geq t_0$:
        \begin{equation*}
            \Exp{c_t(p^{t}) - c_t(q^{t})} \leq \frac{4GKC^{1 - \delta/3} + \frac{6G^{2}N^{3}}{C^{1-\delta/3}}}{(C + t -1)^{\delta/3}} =: \frac{K'}{(C + t - 1)^{\delta/3}}
        \end{equation*}
        Summing over the time steps $\{t_0, \dots, T\}$ we get:
        \begin{align*}
            \sum_{t = t_0}^{T} \Exp{c_t(p^{t}) - c_t(q^{t})} \leq \sum_{t = t_0}^{T} \frac{K'}{(C + t - 1)^{\delta/3}}
        \end{align*}
        Therefore:
        \begin{align*}
            \Exp{\text{Regret}_{D}(T)} &= \sum_{t=1}^{t_0 - 1} \Exp{c_t(p^{t}) - c_t(q^{t})} + \sum_{t=t_0}^{T} \Exp{c_t(p^{t}) - c_t(q^{t})} \\
            &\leq \sum_{t=1}^{t_0 - 1} \Exp{c_t(p^{t})} + \sum_{t = t_0}^{T} \frac{K'}{(C + t - 1)^{\delta/3}} \\
            &\leq \sum_{t=1}^{t_0 - 1} \sum_{i=1}^{N} \Exp{\frac{1}{p_i^t} \norm{g_i^t}_2^2} + K' \int_{t=t_0-1}^{T} \frac{1}{(C + t - 1)^{\delta/3}} dt \\
            &\leq (t_0 - 1)\frac{NG^2}{\varepsilon_{t_0}} + K' \frac{(C + T - 1)^{1 - \delta/3} - (C + t_0 - 2)^{1-\delta/3}}{1 - \delta/3} \\
            &\leq (2^{3/\delta} - 1)\frac{N^2G^2}{\varepsilon_{t_0}} +  K' \frac{(C + T - 1)^{1 - \delta/3} - (C-1)^{1-\delta/3}}{1 - \delta/3}\\
            &= \mathcal{O}(T^{1-\delta/3})
        \end{align*}
        Where the line before the last follows from the fact that $1 \leq t_0 \leq  (2^{3/\delta} - 1)N + 1$ since $C \geq N$.
    \end{proof}
    
    \section{A new mini-batch estimator}
    \label{appendix3}
    \subsection{A class of unbiased estimators}
    It will be useful for our discussion to consider the following class $\mathcal{C}(p^{t, 1}, \dots, p^{t, m})$ of estimators:
    \begin{equation*}
        \hat{g}_b^t(p^{t, 1}, \dots, p^{t, m}):= \frac{1}{m}\sum_{j=1}^{m}\hat{g}_j^t(p^{t, j}) \quad\quad \hat{g}_j^t(p^{t,j}) := \left[\frac{1}{p^{t, j}_{I_t^j}}g_{I_t^j}^{t} + \sum_{k=1}^{j-1} g_{I_t^k}^t\right]
    \end{equation*}
    where each $p^{t, j}$ is a distribution on $[N]\setminus \{I_t^1, \dots, I_t^{j-1}\}$.
    The estimator we proposed in Section \ref{sec5} is:
    \minibatchestimator*
    where the indices $S_t = \{I_t^1, \dots, I_t^m\}$ are sampled without replacement according to $p^t$.
    Setting:
    \begin{equation*}
        p^{t, j}_i := \begin{cases}
        0 &\text{if $i \in \{I_t^1, \dots, I_t^{j-1}\}$} \\
        q_i^{t,j} &\text{otherwise}
        \end{cases}
    \end{equation*}
    we see that our proposed estimator belongs to the class of estimators $\mathcal{C}$ introduced above. The proofs of (a) and (b) of proposition \ref{propertiesminibatchestimator} below apply with almost no modification to any estimator in the class $\mathcal{C}$. A natural question then is which estimator in the above-defined class achieves minimum variance. We answer this in the proof of part $(c)$  below, and show that our proposed estimator (\ref{minibatchestimator}) with $p^t := \argmin_{p \in \Delta}\{c_t(p)\}$ achieves minimum variance.
    
    \subsection{Proof of proposition \ref{propertiesminibatchestimator}}
    \propertiesminibatchestimator*
    \begin{proof}
    ~
        \begin{enumerate}[\upshape (a)]
        \item For $j \in [m]$ and conditional on $S_t^{j-1}$ we have:
        \begin{align*}
            \Exp{\hat{g}_j^t} &= \sum_{\substack{i = 1 \\ i \notin S_t^{j-1}}}^{N} q^{t, j}_{i} \left[\frac{1}{q_i^{t,j}}g_i^t + \sum_{k \in S_t^{j-1}}g_{k}^t\right] \\
            &= \sum_{\substack{i = 1 \\ i \notin S_t^{j-1}}}^{N} g_i^t + \left(\sum_{k \in S_t^{j-1}}g_{k}^t\right) \underbrace{\left(\sum_{\substack{i = 1 \\ i \notin S_t^{j-1}}}^{N} q_i^{t,j}\right)}_{= 1} \\
            &= \sum_{i=1}^{N} g_i^t\\
            &= g^t
        \end{align*}
        Taking expectation with respect to the choice of $S_t^{j-1}$, and taking the average over $j \in [m]$ we get the result.
        
        \item We have:
        \begin{align*}
            \Exp{\norm{\hat{g}_b^t - g^t}_2^2} &= \Exp{\norm{\frac{1}{m}\sum_{j=1}^{m}\hat{g}_j^t - g^t}_2^2} \\
            &= \frac{1}{m^2}\sum_{j=1}^{m} \Exp{\norm{\hat{g}_j^t - g^t}_2^2} + \frac{2}{m^2} \sum_{\substack{j, i\\ j < i}}^{m} \Exp{\langle \hat{g}_j^t - g^t, \hat{g}_i^t - g^t\rangle}
        \end{align*}
        To show the claim, it is therefore enough to show that second term is zero. Let $j \in [m-1]$. Conditional on $S_t^{i-1}$ we have:
        \begin{equation*}
            \Exp{\langle \hat{g}_j^t - g^t, \hat{g}_i^t - g^t\rangle} = \langle \hat{g}_j^t - g^t, \Exp{\hat{g}_i^t} - g^t\rangle = 0
        \end{equation*}
        where we used the fact that the conditional expectation is zero by part (a). Taking expectation with respect to $S_t^{i-1}$ on both sides yields the result.
        \item As discussed in the previous section, (b) applies to all estimators in the class $\mathcal{C}$, so we have for all such estimators:
        \begin{equation*}
            \Exp{\norm{\hat{g}_b^t(p^{t, 1}, \dots, p^{t, m}) - g^t}_2^2} = \frac{1}{m^2}\sum_{j=1}^{m} \Exp{\norm{\hat{g}_j^t(p^{t, j}) - g^t}_2^2}
        \end{equation*}
        minimizing over $(p^{t,1}, \dots, p^{t,m})$ by minimizing each term with respect to its variable we get:
        \begin{equation*}
            \argmin_{(p^{t, 1}, \dots, p^{t, m})}\{\Exp{\norm{\hat{g}_b^t(p^{t, 1}, \dots, p^{t, m}) - g^t}_2^2}\} = \left(\frac{p^{t, *}}{1 - \sum_{k=1}^{j-1}p_{I_t^k}^{t, *}}\right)_{j=1}^{m}
        \end{equation*}
        where:
        \begin{equation*}
            p^{t, *} = \argmin_{p \in \Delta}\{c_t(p)\} = \frac{\norm{g_i^t}_2}{\sum_{j=1}^{N}\norm{g_j^t}_2}
        \end{equation*}
        Recalling that our estimator is in $\mathcal{C}$, and noticing that the optimal probabilities over $\mathcal{C}$ are feasible for our estimator we get the result.
        
        \item Fix $t \in [T]$. We drop the superscript $t$ from $p^t$, $q_{i}^{t,j}$, and $g_i^t$. We also drop the subscript $t$ from $I_{t}^{j}$ and $S_t^j$ to simplify notation.
        Define for $j \in [m]$:
        \begin{equation*}
            x_j := \frac{1}{q_{I^j}^{j}} g_{I^j} \quad\quad \mu_j := g^{t} - \sum_{k=1}^{j-1} g_{I^k}
        \end{equation*}
        We have from part (a):
        \begin{equation*}
            \Exp{x_j} = \mu_j
        \end{equation*}
        Before proceeding with the proof, we first derive an identity relating $q_i^{j+1}$ and $q_{i}^{j}$:
        \begin{align*}
            \frac{1}{q_i^{j+1}} &= \frac{1 - \sum_{k\in S^j} p_k}{p_i} \\
            &= \frac{1 - \sum_{k\in S^{j-1}} p_k - p_{I^{j}}}{p_i} \\
            &= \left(1 - \frac{p_{I^j}}{1 - \sum_{k\in S^{j-1}} p_k}\right)\left(\frac{1 - \sum_{k\in S^{j-1}} p_k}{p_i}\right) \\
            &= \left(1 - q_{I^j}^j\right) \frac{1}{q_i^j}
        \end{align*}
        Now, conditional on $S_t^{j}$, we have:
        \begin{align*}
            &\Exp{\norm{\hat{g}_{j+1}^t - g^t}_2^2} \\
            &= \Exp{\norm{x_{j+1} - \mu_{j+1}}_2^2} \\
            &= \Exp{\norm{x_{j+1}}_2^2} - \norm{\mu_{j+1}}_2^2 \\
            &= \left(\sum_{\substack{i=1\\ i \notin S^{j}}}^{N}
            \frac{1}{q_i^{j+1}} \norm{g_i}_2^2\right) - \norm{\mu_{j+1}}_2^2 \\
            &= \left(\sum_{\substack{i=1\\ i \notin S^{j-1}}}^{N}
            \frac{1}{q_i^{j+1}} \norm{g_i}_2^2\right) - \frac{1}{q_{I^j}^{j+1}}\norm{g_{I^j}}_2^2 - \norm{\mu_{j} - g_{I^{j}}}_2^2\\
            &= \left(1 - q_{I^j}^j\right) \left(\sum_{\substack{i=1\\ i \notin S^{j-1}}}^{N}
            \frac{1}{q_i^{j}} \norm{g_i}_2^2\right) - \left(1 - q_{I^j}^j\right)\frac{1}{q_{I^j}^{j}}\norm{g_{I^j}}_2^2 - \left(\norm{\mu_j}_2^2 - 2 \langle \mu_j, g_{I^j} \rangle  + \norm{g_{I^j}}_2^2\right) \\
            &= \left(1 - q_{I^j}^j\right) \left(\sum_{\substack{i=1\\ i \notin S^{j-1}}}^{N}
            \frac{1}{q_i^{j}} \norm{g_i}_2^2 - \norm{\mu_j}_2^2\right) - \left(\frac{1}{q^{j}_{I^j}} \norm{g_{I^j}}_2^2 - 2\langle g_{I^j}, \mu_j \rangle + q_{I^j}^{j} \norm{\mu_j}_2^2 \right) \\
            &=  \left(1 - q_{I^j}^j\right) \Exp{\norm{x_j - \mu_j}_2^2} - q_{I^j}^{j}  \norm{x_j - \mu_j}_2^2 \\
            &= \left(1 - q_{I^j}^j\right) \Exp{\norm{\hat{g}_{j}^t - g^t}_2^2} - q_{I^j}^{j}  \norm{\hat{g}_{j}^t - g^t}_2^2
        \end{align*}
        where the expectation in the last two lines is conditional on $S^{j-1}$. Taking expectation with respect to $S^{j-1}$ on both sides yields the result.
        \end{enumerate}
    \end{proof}
    
    \section{Extension to constant step-size SGD}
    \label{appendix4}
    While our analysis heavily relies on the assumption of decreasing step-sizes, we have found empirically that a slight modification of our method works just as well when a constant step-size is used. We propose the following epsilon sequence to account for the use of a constant step-size:
    \begin{equation}
        \varepsilon_t = \frac{1}{C^{1-\delta/3}  (C + m(t-1))^{\delta/3}} + p_{min}
        \label{eq:constantminibatchepsilonsequence}
    \end{equation}
    for a constant $p_{min} \in [0, 1/N]$ and the following condition on $C$:
    \begin{equation*}
        C \leq \frac{1}{\frac{1}{N} - p_{min}}
    \end{equation*}
    which ensures that $\varepsilon_1 \leq 1/N$.
    We ran the same experiment on MNIST, IJCNN1, and CIFAR10 as in Section \ref{sec6}, but with a constant step-size $\alpha_t = \alpha = \frac{m}{2NL}$, and the epsilon sequence (\ref{eq:constantminibatchepsilonsequence}) with $p_{min} = \frac{1}{5N}$ and $C = \frac{1}{\frac{1}{N} - p_{min}}$, and $\delta = 1$. The results are displayed in figure \ref{fig:constant_real}, showing a similar performance compared to the decreasing step-sizes case. Note that choosing a too small $p_{min}$ can start to deteriorate the performance of the algorithm. It is still unclear how to set $p_{min}$ so as to guarantee good performance, but our experiments suggest that setting $\frac{1}{5N}$ is a safe choice.
    
    \begin{figure}
      \centering
      \includegraphics[width=5.0in,keepaspectratio]{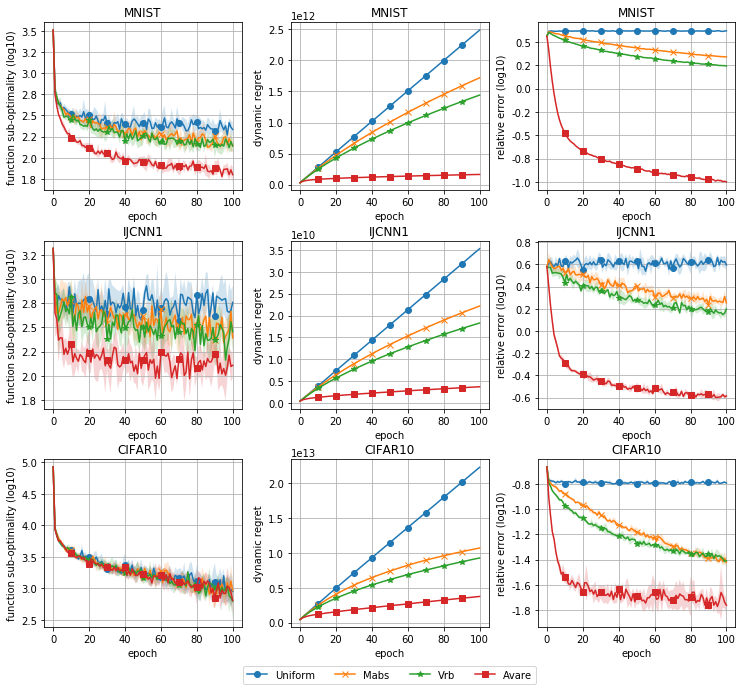}
      \caption{Comparison of the performance of importance samplers on an l2-regularized softmax regression model on three real world datasets: MNIST (top), IJCNN1 (middle), CIFAR10 (bottom). For this set of experiments, SGD was run using a constant step size.}
      \label{fig:constant_real}
    \end{figure}
    
    \newpage
    \bibliographystylesupp{plain}
    \bibliographysupp{mendeley}
\end{appendices}

\end{document}